\documentclass{article}

\usepackage{PRIMEarxiv}

\usepackage[utf8]{inputenc} 
\usepackage[T1]{fontenc}    
\usepackage{url}            
\usepackage{booktabs}       
\usepackage{array}          
\usepackage{amsfonts}       
\usepackage{nicefrac}       
\usepackage{microtype}      
\usepackage{fancyhdr}       
\usepackage{graphicx}       
\graphicspath{{media/}}     

\usepackage{amsmath} 
\usepackage{amsthm}
\usepackage{amssymb}
\usepackage{longtable}
\usepackage{mathtools}
\usepackage{pdflscape}
\newcolumntype{L}[1]{>{\raggedright\arraybackslash}p{#1}}
\newtheorem{theorem}{Theorem}[section]
\newtheorem{corollary}{Corollary}[theorem]

\newtheorem{proposition}[theorem]{Proposition}
\newtheorem{definition}{Definition}

\newtheorem{example}{Example}

\usepackage{algorithm}
\usepackage{algpseudocode}
\usepackage{natbib}
\usepackage{hyperref}       

\title{A Fiber Criterion for Representation Identifiability in Supervised Learning}

\author{
  Vasileios Sevetlidis\\
  Athena Research Center, Kimmeria Campus, Xanthi, Greece \\
  Democritus University of Thrace, Vas. Sofias Campus, Xanthi, Greece \\
  International Hellenic University, Serres, Greece\\
  \texttt{vasiseve@athenarc.gr} \\
}

\begin{document}
\maketitle

\begin{abstract}
Supervised learning evaluates predictors through their input-output behavior. When a predictor is implemented as a composition $f=c\circ h$, supervised evidence constrains the composite map $f$ but need not determine the representation-head factorization $(h,c)$. This paper formalizes the resulting representation-level identifiability problem: for a class of admissible representation-head pairs, a representation property is identifiable from the induced predictor exactly when it is constant on the fibers of the projection $(h,c)\mapsto c\circ h$, equivalently when it descends to a well-defined property of the predictor. Predictor-preserving augmentation gives a canonical obstruction: auxiliary information can be appended to a representation while the head ignores it, leaving the predictor unchanged but altering properties such as minimality, compression, invariance, equivariance, nuisance information, or semantic accessibility. This construction separates representation identifiability from optimization and finite-sample estimation. Finite-sample diagnostics illustrate, rather than prove, the criterion: exact algebraic witnesses hold the predictor fixed while changing representation diagnostics, and matched-performance Waterbirds models show that different constraints can select different representations at similar supervised performance. The results clarify that representation-level claims require assumptions, objectives, measurements, or inductive biases beyond supervised predictive behavior alone.
\end{abstract}

\keywords{representation learning, identifiability, supervised learning, invariance, nuisance information, representation probing}

\section{Introduction}

Supervised learning evaluates predictors through their input--output behavior. 
A model is judged by whether its predictions match the target labels, or by how
small its supervised loss is. Modern supervised models, however, are often
interpreted not only as predictors but also as representation learners
\citep{bengio2013representation}. A predictor is commonly implemented as a
composition $f=c\circ h$, where $h$ maps inputs to an internal representation
and $c$ maps that representation to an output. The representation is then
described as compressed, invariant, nuisance-free, semantically meaningful,
disentangled, or aligned with task-relevant factors
\citep{tishby2015deep, bengio2013representation, higgins2017beta,
kim2018disentangling, locatello2019challenging, scholkopf2021toward}.
These are claims about the internal representation, not only about the final
input--output map.

This paper studies which such claims are determined by supervised predictive
behavior. The basic issue is that supervised evidence constrains the composite
predictor $c\circ h$, but need not determine the representation--head
factorization $(h,c)$ that realizes it. Two admissible factorizations may
induce exactly the same predictor while having substantially different
representations. One representation may discard nuisance information while
another retains it; one may be compressed while another is not; one may be
invariant to a transformation while another encodes transformation-sensitive
information ignored by the head. These differences are invisible at the level
of the induced supervised predictor. This is closely related in spirit to
known non-identifiability phenomena in representation learning, where additional
assumptions or inductive biases are needed to recover particular latent or
semantic structure from observational evidence alone
\citep{locatello2019challenging, scholkopf2021toward}.

This gap is formulated as an identifiability problem. Given a class of
admissible representation--head pairs, the projection
\[
\Pi(h,c)=[c\circ h]_{P_X}
\]
maps each factorization to its induced predictor, identified up to
$P_X$-almost-sure equality. All factorizations in the same fiber of $\Pi$
are observationally equivalent from the supervised predictive point of view.
A representation-level property is therefore identifiable from the induced
predictor exactly when it is constant on these fibers. Equivalently, the
property must descend to a well-defined property of the predictor itself.
If a property varies among admissible factorizations of the same predictor,
then supervised predictive behavior alone cannot determine it.

A simple construction gives a canonical obstruction. Starting from a
factorization $(h,c)$, append auxiliary information $q(x)$ to the
representation and extend the head so that it ignores the added coordinate:
\[
h^q(x)=(h(x),q(x)),
\qquad
c^q(u,v)=c(u).
\]
Then $c^q\circ h^q=c\circ h$ pointwise, so the induced predictor and
supervised risk are unchanged. Nevertheless, the representation may now contain
additional nuisance information, fail to be minimal, lose compression, violate
an invariance property, or make a semantic attribute accessible. Thus the
predictor is fixed while representation-level properties can change. This
obstruction also echoes empirical concerns that high predictive performance
need not imply reliance on robust, semantically intended, or causally relevant
features \citep{zhang2017understanding, geirhos2020shortcut}.

The claim is identifiability-theoretic: supervised predictive behavior
certifies a representation-level property only when that property is determined
by the induced predictor. Invariance, compression, and semantic structure may
arise through architectures, regularization, objectives, augmentation,
optimization bias, or measurement, but they are not consequences of supervised
prediction alone. When representation-level properties are obtained, their
justification must come from additional structure, such as architectural
restrictions, regularization, bottlenecks, data augmentation, auxiliary losses,
multiple environments, causal assumptions, optimization bias, or direct
representation-level measurement
\citep{tishby2015deep, arjovsky2019invariant, locatello2019challenging,
geirhos2020shortcut, scholkopf2021toward}.

The contributions are as follows. \textbf{First}, the paper formulates
representation-level identifiability for supervised predictors through the
projection
\[
\Pi(h,c)=[c\circ h]_{P_X}.
\]
\textbf{Second}, it proves a fiber/descent criterion: a representation property
is identifiable from the induced predictor if and only if it is constant on
every fiber of $\Pi$, equivalently if it descends to a well-defined property of
the predictor. \textbf{Third}, it gives a general predictor-preserving
augmentation obstruction and applies it to common representation-level
desiderata, including minimality, compression, invariance, equivariance,
nuisance information, and semantic accessibility. The empirical component
provides finite-sample diagnostics of fiber variation; the theorem itself is
exact. Exact algebraic witnesses demonstrate finite-sample variation of common
representation diagnostics inside a fixed predictor fiber, while a
matched-performance Waterbirds study illustrates how different
representation-level constraints can select different diagnostics at similar
supervised performance.

The rest of the paper follows the progression from formal obstruction to
empirical diagnostics. Section~\ref{sec:theory} introduces the supervised
factorization setup, defines predictor equivalence and representation-level
identifiability, proves the fiber/descent criterion, and derives the
predictor-preserving augmentation obstruction. Section~\ref{sec:examples}
applies the criterion to common representation-level properties.
Section~\ref{sec:experiments} presents the empirical evidence, first through
exact predictor-preserving diagnostic witnesses in
Section~\ref{subsec:experiments-exact-witnesses}, and then through the
Waterbirds matched-performance study in
Section~\ref{subsec:experiments-constraints}. Section~\ref{sec:discussion-limitations}
discusses scope, interpretation, and limitations, concluding with 
Section~\ref{sec:conclusion}. Appendix~\ref{app:conceptual-figures}
provides additional conceptual illustrations of the fiber obstruction.
Appendix~\ref{app:experiments} gives implementation details, controls, and
additional results for the exact predictor-preserving witnesses.
Appendix~\ref{app:additional-experimental-details} reports additional
Waterbirds details, predictor-preserving augmentation details, and near-fiber
diagnostics.

\section{Background and Related Work}
\label{sec:related-work}

Modern supervised learning evaluates models through their input--output
behavior, but learned models are often interpreted through their internal
representations. A predictor is commonly implemented as a factorization
$f=c\circ h$, where $h$ maps inputs to an internal representation and $c$
maps representations to predictions. Representation learning is motivated by
the idea that useful internal variables can expose factors of variation that
support prediction, transfer, abstraction, and generalization
\citep{bengio2013representation}. In this view, the representation is not only
a computational intermediate. It is also treated as an object that may be
compressed, invariant, equivariant, disentangled, semantically meaningful,
nuisance-free, or causally aligned
\citep{tishby1999information,tishby2015deep,higgins2017beta,
locatello2019challenging,scholkopf2021toward}.

These representation-level claims are not ordinary claims about predictive
accuracy. Supervised loss evaluates the composite map $c\circ h$, whereas
properties such as compression, invariance, nuisance removal, semantic
accessibility, or causal alignment concern the particular internal variable
$h(X)$ through which the predictor is realized. Thus two questions must be
distinguished. One may ask whether a model predicts well. One may also ask
whether the representation used by that model has a particular structural
property. The first question concerns the induced predictor; the second
concerns a factorization of that predictor.

This distinction is related to the broader phenomenon of underspecification in
modern machine learning. An underspecified learning pipeline may return many
predictors with comparable in-distribution performance, while those predictors
can behave differently under deployment shifts
\citep{damour2022underspecification}. Shortcut-learning and robustness studies
likewise show that high predictive performance does not by itself imply
reliance on intended, stable, or semantically robust features
\citep{geirhos2020shortcut,sagawa2020distributionally}. Most such work focuses
on ambiguity among predictors, training procedures, or deployment behavior. The
question studied here is sharper and more internal: even after the composite
supervised predictor $c\circ h$ is fixed, which properties of the
representation component $h$ are determined? The obstruction addressed in this
paper is therefore not only statistical, finite-sample, or
distribution-shift-based. It is an algebraic ambiguity of the
representation--head factorization itself.

Several lines of work add objectives or architectural restrictions precisely in
order to shape this internal factorization. The information bottleneck
principle seeks representations that preserve information about the target
while compressing information about the input
\citep{tishby1999information,tishby2015deep}. Variational and neural versions
make this trade-off trainable in modern models \citep{alemi2017deep}, and
related information-theoretic analyses connect compression, invariance, and
generalization \citep{achille2018emergence,achille2018information}. From the
perspective of the present paper, bottleneck methods are important because
they supply additional structure that is not contained in the supervised
predictor alone. They restrict or penalize the admissible representations,
thereby selecting factorizations with desired compression or minimality
properties.

A related family of methods encourages representations to remove nuisance
information or remain stable across domains, environments, or transformations.
Domain-adversarial training promotes task-predictive but domain-uninformative
representations \citep{ganin2016domain}. Invariant risk minimization and
related objectives use multi-environment structure to encourage stable
predictive mechanisms
\citep{arjovsky2019invariant,ahuja2020invariant,krueger2021out,
li2022invariant}, while subsequent work has analyzed limitations of practical
invariance objectives \citep{kamath2021does}. These methods are motivated by
the fact that empirical risk minimization may exploit environment-specific or
spurious correlations \citep{geirhos2020shortcut,sagawa2020distributionally}.
In the language of this paper, multiple environments, adversarial losses, and
invariance penalties enrich the learning problem beyond the single supervised
predictor. They can remove or penalize factorizations that retain nuisance
information, but this information is not excluded merely by the fact that the
final predictor has low supervised loss.

Architectural approaches impose structure more directly. Group-equivariant
convolutional networks and geometric deep learning build transformation
structure into the hypothesis class itself
\citep{cohen2016group,bronstein2021geometric}. Augmentation-based and
self-supervised methods impose relationships between transformed views
\citep{chen2020simple,he2020momentum,grill2020bootstrap}. Theoretical analyses
of self-supervised learning also formalize idealized representations through
augmentation-induced equivalence relations and characterize when such
representations support downstream tasks invariant to those transformations
\citep{dubois2022improving}. These works show how representation properties
can be made stable by design, by restricting the class of admissible maps or
by adding pretext constraints. The present paper takes a complementary
viewpoint: it asks what can be inferred before such restrictions are imposed.
When equivariance, invariance, or transformation stability holds, the
justification comes from architectural, augmentation, or objective-level
structure, not from supervised predictor equivalence alone.

The disentanglement and causal representation literatures make the role of
additional assumptions especially explicit. Methods such as $\beta$-VAE and
FactorVAE encourage factorized latent coordinates
\citep{higgins2017beta,kim2018factorvae}. Negative results show that
unsupervised disentanglement is not identifiable without inductive bias
\citep{locatello2019challenging}, while positive identifiability results
typically rely on auxiliary variables, temporal structure, labels, conditional
priors, or restrictions on the generative process
\citep{hyvarinen2019nonlinear,khemakhem2020variational}. Causal representation
learning likewise emphasizes that recovering causally meaningful variables
requires assumptions beyond observational predictive success
\citep{scholkopf2021toward}. These works usually study the recovery of latent
generative factors or causal variables. The present paper studies a different
object: not whether a true latent representation is recovered, but whether a
property of the learned representation component $h$ is determined by the
supervised predictor $c\circ h$. The common lesson is nevertheless the same:
internal structure does not follow from observational or predictive success
without additional assumptions.

A more directly related line of work studies identifiability of learned
representations themselves. \citet{roeder2021linear} analyze conditions under
which representation functions learned for downstream tasks are identifiable
up to a linear transformation. Recent work also distinguishes statistical
identifiability, concerning consistency of representations across runs or
solutions, from structural identifiability, concerning alignment with
unobserved ground-truth factors \citep{nelson2026statistical}. These
approaches ask whether learned representations are recoverable, comparable
across solutions, or determined up to a specified transformation class. The
question studied here is different. The observable is taken to be the induced
supervised predictor $c\circ h$. The object of interest is a
representation-level property, such as minimality, compression, invariance,
nuisance information, or semantic accessibility. The paper asks whether such a
property is well defined on the equivalence class of all admissible
factorizations that induce the same predictor.

This gives a different role to equivalence classes. In invariant and
self-supervised representation learning, equivalence classes often group inputs
or augmented views that should share a representation
\citep{chen2020simple,dubois2022improving}. In geometric deep learning,
symmetry structure is imposed through group actions, equivariant maps, and
geometric constraints on the hypothesis class
\citep{cohen2016group,bronstein2021geometric}. Here, the relevant equivalence
classes are not sets of equivalent inputs and not geometric fibers inside a
fixed architecture. They are fibers in the space of representation--head
factorizations over a fixed supervised predictor. Two pairs $(h,c)$ and
$(\tilde h,\tilde c)$ lie in the same fiber when
\[
c\circ h = \tilde c\circ \tilde h
\quad P_X\text{-a.s.}
\]
The central issue is whether a property of $h$ is constant over all such
factorizations. If it is not, then the property cannot be inferred from the
induced predictor.

Empirical studies of neural representations provide a complementary
perspective. Representation-similarity methods such as SVCCA and CKA compare
internal representations across layers, random seeds, tasks, and architectures
\citep{raghu2017svcca,kornblith2019similarity}. Recent surveys distinguish
representational similarity, which compares internal activations, from
functional similarity, which compares input--output behavior
\citep{klabunde2025similarity}. Pathwise and geometric diagnostics provide a
further perspective, measuring aspects of representation geometry that may not
be captured by pointwise similarity alone
\citep{sevetlidis2026gaugeinvariant}. These diagnostics compare or measure
chosen representations. The present paper instead asks a prior identifiability
question: whether the measured representation property is determined by the
predictor at all. Two systems can be functionally equivalent as predictors
while differing as representation--head factorizations.

Probing methods make this distinction especially concrete. A probe measures
what information is accessible from a learned representation, for example
through a linear classifier trained on frozen features
\citep{alain2017understanding}. Later work emphasizes that probe conclusions
depend on probe capacity, controls, selectivity, and the distinction between
information being present in a representation and information being used by
the original predictor
\citep{hewitt2019designing,belinkov2022probing}. Conditional probing and
usable-information approaches sharpen this concern by measuring information
available beyond a baseline representation
\citep{hewitt2021conditional}. The obstruction studied here gives an exact
algebraic form of the same dilemma. A statistic can be appended to a
representation and made easily decodable by a probe while the original head
ignores it completely. Thus probe accessibility is a property of the chosen
representation, not automatically a property of the supervised predictor.

Prior work shows that representation-level structure is
usually obtained or justified by adding something beyond plain supervised
prediction: bottlenecks, regularizers, adversarial objectives, augmentation
schemes, equivariant architectures, auxiliary variables, multiple
environments, causal assumptions, optimization bias, or direct
representation-level measurements
\citep{tishby1999information,alemi2017deep,ganin2016domain,
arjovsky2019invariant,cohen2016group,locatello2019challenging,
scholkopf2021toward,hewitt2021conditional}. The baseline identifiability question for supervised factorized predictors: \textit{"which properties of an internal representation are uniquely determined by the supervised predictor itself?"}, however, precedes these interventions. The contribution of this paper is
to formalize the baseline obstruction that makes such additions necessary. For
a class of admissible representation--head pairs, consider the projection
\[
\Pi(h,c)=[c\circ h]_{P_X}.
\]
A representation property is identifiable from the induced predictor exactly
when it descends through this projection, equivalently when it is constant on
the fibers of $\Pi$. Predictor-preserving augmentation gives a canonical
witness of failure: one may replace $h(x)$ by $(h(x),q(x))$ and extend the head
so that it ignores $q(x)$, leaving $c\circ h$ unchanged while altering
compression, invariance, nuisance information, equivariance, semantic
accessibility, or probe behavior. The resulting criterion provides a reference
point for interpreting representation-level claims. When a property does not
descend to the induced predictor, evidence for that property must enter
through additional structure.

\section{Theory}
\label{sec:theory}

Supervised prediction constrains the induced input--output map, but not in
general the representation--head factorization that realizes it. This distinction
is consistent with the general identifiability viewpoint that observations
determine equivalence classes of latent or structural objects rather than
necessarily determining those objects uniquely
\citep{paulino1994identifiability, huang2016modelidentifiability}. This section
formalizes that distinction through predictor equivalence and representation
identifiability. The central result is a fiber/descent criterion: a
representation property is identifiable from the induced predictor exactly when
it is constant over all admissible factorizations inducing that predictor,
equivalently when it descends to a well-defined property of the composite
predictor. The criterion yields an augmentation obstruction: any representation
property that can be changed by appending unused admissible information to the
representation is not identifiable from the induced predictor, and therefore is
not identifiable from scalar supervised risk alone.

The hierarchy underlying the section is shown in Figure~\ref{fig:identifiability-hierarchy}. A representation--head factorization $(h,c)$ determines the induced predictor $f=c\circ h$, and the induced predictor determines the scalar risk $R_P(f)$. Each arrow forgets information. Consequently, representation-level claims require assumptions beyond predictor-level claims, and  predictor-level claims require assumptions beyond scalar risk comparisons. This mirrors the role of inductive biases in representation learning: semantic, disentangled, invariant, or causal structure is generally not recovered from predictive or observational evidence without additional assumptions \citep{bengio2013representation, locatello2019challenging, scholkopf2021toward}.

\begin{figure}[t]
\centering
\includegraphics[width=0.86\linewidth]{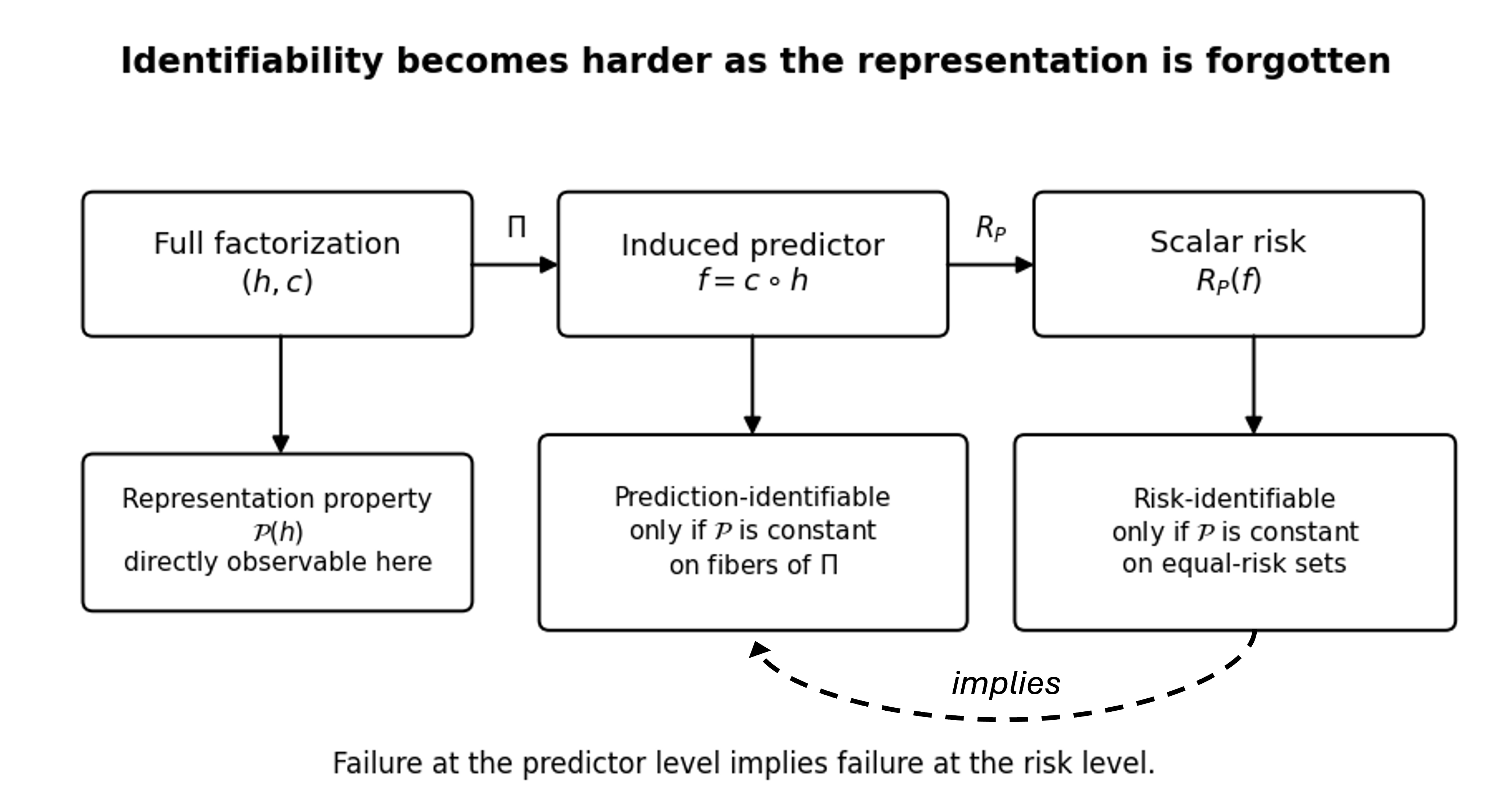}
\caption{Information loss from factorization to prediction to scalar risk. The full representation--head pair $(h,c)$ determines the induced predictor $f=c\circ h$, which in turn determines the scalar risk $R_P(f)$. Representation-level properties are properties of the factorization. Failure of identifiability at the predictor level therefore implies failure of identifiability from scalar supervised risk.}
\label{fig:identifiability-hierarchy}
\end{figure}

\subsection{Supervised risk and factorized predictors}
\label{subsec:supervised-risk}

Let $(\mathcal X,\mathcal A_{\mathcal X})$ be a measurable input space and let $(\mathcal Y,\mathcal A_{\mathcal Y})$ be a measurable output space. Let $P$ be a probability measure on $\mathcal X\times\mathcal Y$, and let $(X,Y)\sim P$. Denote by $P_X$ the marginal distribution of $X$.

Let $(\mathcal A,\mathcal A_{\mathcal A})$ be a measurable action space. A predictor is a measurable map
\begin{equation}
f:\mathcal X\to\mathcal A.
\end{equation}
Given a measurable loss function $\ell:\mathcal A\times\mathcal Y\to[0,\infty]$, the population risk of $f$ is
\begin{equation}
R_P(f)
=
\mathbb E_P\big[\ell(f(X),Y)\big],
\end{equation}
where the expectation is understood as an extended nonnegative expectation. This is the standard statistical decision-theoretic formulation in which actions are evaluated by expected loss or risk \citep{berger1985statistical,miescke2008statistical}. Throughout this section, identifiability from supervised risk refers only to
identifiability from this scalar population-risk value, not from the full
predictor, conditional risks, loss distributions, empirical optimization
dynamics, or algorithm-specific selection effects. This convention separates
statistical decision-theoretic risk evaluation from algorithmic or
optimization-induced model selection \citep{berger1985statistical,
hardt2016train, neyshabur2017exploring}.

A representation--head pair consists of a measurable representation map
\begin{equation}
h:\mathcal X\to\mathcal H
\end{equation}
and a measurable prediction head
\begin{equation}
c:\mathcal H\to\mathcal A,
\end{equation}
where $(\mathcal H,\mathcal A_{\mathcal H})$ is a measurable representation space. The pair $(h,c)$ induces the composite predictor
\begin{equation}
f_{h,c}
=
c\circ h.
\end{equation}
The supervised risk of the pair is defined by
\begin{equation}
R_P(h,c)
:=
R_P(c\circ h)
=
\mathbb E_P\big[\ell(c(h(X)),Y)\big].
\end{equation}

The following elementary fact is the starting point of the obstruction. It states that the supervised objective is a functional of the composite predictor, not of the particular factorization used to realize it.

\begin{proposition}[Risk is a functional of the composite predictor]
\label{prop:risk-composite}
Let $h:\mathcal X\to\mathcal H$ and $\tilde h:\mathcal X\to\widetilde{\mathcal H}$ be measurable representations, and let $c:\mathcal H\to\mathcal A$ and $\tilde c:\widetilde{\mathcal H}\to\mathcal A$ be measurable heads. If
\begin{equation}
c(h(X))=\tilde c(\tilde h(X))
\qquad
P_X\text{-a.s.},
\end{equation}
then
\begin{equation}
R_P(c\circ h)=R_P(\tilde c\circ \tilde h).
\end{equation}
\end{proposition}

\begin{proof}
The assumption implies
\begin{equation}
c(h(X))=\tilde c(\tilde h(X))
\qquad
P\text{-a.s.}
\end{equation}
Therefore
\begin{equation}
\ell(c(h(X)),Y)
=
\ell(\tilde c(\tilde h(X)),Y)
\qquad
P\text{-a.s.}
\end{equation}
Taking extended nonnegative expectations gives the claim.
\end{proof}

Proposition~\ref{prop:risk-composite} separates two levels of non-identifiability. Scalar risk generally does not identify the predictor, since distinct predictors may have the same risk. Moreover, even before this coarser ambiguity appears, distinct factorizations of the same predictor are already indistinguishable to supervised risk.

\begin{example}[Same composite predictor, same risk]
\label{ex:same-composite-same-risk}
Let $X=(U,V)$, and suppose the prediction target depends on $U$ alone. Consider two factorizations:
\begin{equation}
h_1(X)=U,
\qquad
c_1(u)=a(u),
\end{equation}
and
\begin{equation}
h_2(X)=(U,V),
\qquad
c_2(u,v)=a(u).
\end{equation}
Then
\begin{equation}
c_1(h_1(X))=a(U)=c_2(h_2(X))
\qquad
P_X\text{-a.s.}
\end{equation}
Therefore, for any supervised loss $\ell$,
\begin{equation}
R_P(c_1\circ h_1)=R_P(c_2\circ h_2).
\end{equation}
The two representations may differ substantially: $h_1$ discards $V$, whereas $h_2$ retains it. Nevertheless, supervised risk cannot distinguish the two factorizations because their final predictions are identical.
\end{example}

\subsection{Predictor equivalence}
\label{subsec:predictor-equivalence}

Let $\mathcal M$ be a class of admissible representation--head pairs. Each element of $\mathcal M$ is a pair $(h,c)$, where $h:\mathcal X\to\mathcal H_h$ is a measurable representation into some measurable space $\mathcal H_h$, and $c:\mathcal H_h\to\mathcal A$ is a measurable prediction head. The representation space may depend on the pair.

Since supervised prediction under $P_X$ is insensitive to changes on
$P_X$-null sets, predictors are regarded modulo $P_X$-almost-sure equality.
Let $[f]_{P_X}$ denote the equivalence class of a measurable predictor
$f:\mathcal X\to\mathcal A$ under $P_X$-almost-sure equality. The induced
predictor class is
\begin{equation}
\Pi(\mathcal M)
=
\{[c\circ h]_{P_X}:(h,c)\in\mathcal M\}.
\end{equation}
The projection from a factorization to its induced predictor is
\begin{equation}
\Pi:\mathcal M\to\Pi(\mathcal M),
\qquad
\Pi(h,c)=[c\circ h]_{P_X}.
\end{equation}

Figure~\ref{fig:predictor-equivalence-fibers} represents the fibers of this projection, namely the sets of admissible factorizations that induce the same predictor. The induced predictor is the image point $[c\circ h]_{P_X}$, not the particular point $(h,c)$ inside the fiber. Thus, representation non-identifiability is naturally formulated as variation within fibers of $\Pi$.

\begin{figure}[t]
\centering
\includegraphics[width=0.92\linewidth]{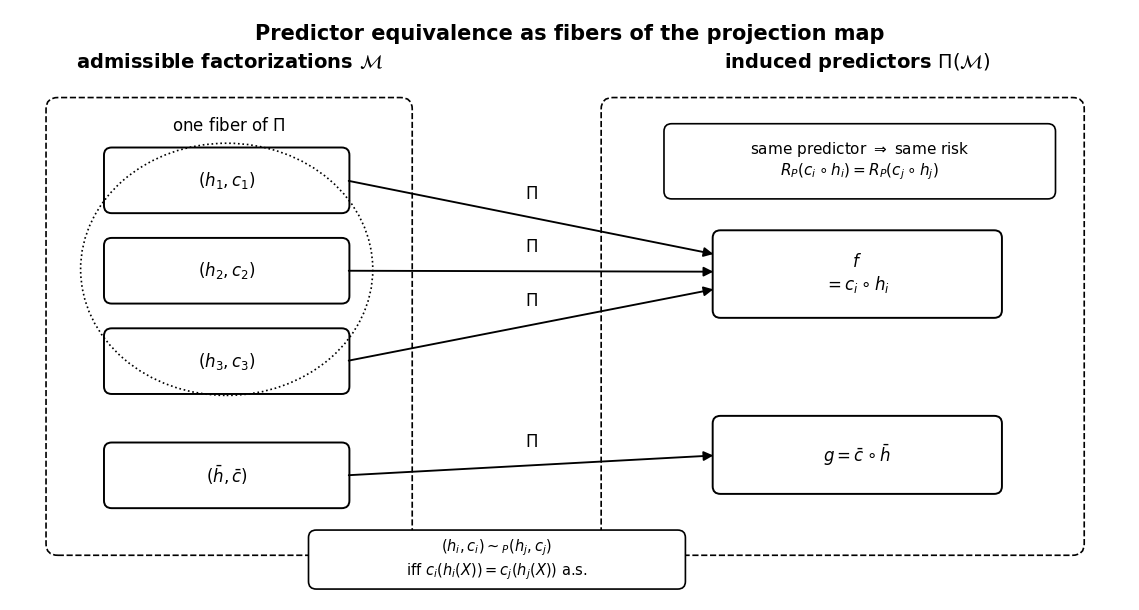}
\caption{Predictor equivalence as fibers of the projection map. The projection $\Pi$ maps a representation--head factorization $(h,c)$ to its induced predictor $[c\circ h]_{P_X}$. Distinct factorizations may induce the same predictor and therefore lie in the same fiber of $\Pi$. Supervised risk is constant on each such fiber, so representation-level properties are identifiable from the induced predictor exactly when they are invariant within these fibers.}
\label{fig:predictor-equivalence-fibers}
\end{figure}

\begin{definition}[Predictor equivalence]
\label{def:predictor-equivalence}
Two admissible pairs $(h,c),(\tilde h,\tilde c)\in\mathcal M$ are predictor-equivalent under $P_X$, written
\begin{equation}
(h,c)\sim_P(\tilde h,\tilde c),
\end{equation}
if
\begin{equation}
c(h(X))=\tilde c(\tilde h(X))
\qquad
P_X\text{-a.s.}
\end{equation}
Equivalently, $(h,c)\sim_P(\tilde h,\tilde c)$ if and only if $\Pi(h,c)=\Pi(\tilde h,\tilde c)$.
\end{definition}

Predictor-equivalent pairs induce the same supervised input--output behavior. The equivalence classes of $\sim_P$ are precisely the fibers of the projection map $\Pi$. This is the same abstract structure that appears in identification problems, where the observation map partitions candidate objects into observationally equivalent classes \citep{paulino1994identifiability,basse2020general}.

\begin{corollary}[Predictor equivalence implies risk equivalence]
\label{cor:prediction-implies-risk}
If $(h,c)\sim_P(\tilde h,\tilde c)$, then
\begin{equation}
R_P(c\circ h)=R_P(\tilde c\circ\tilde h).
\end{equation}
\end{corollary}

\begin{proof}
This is Proposition~\ref{prop:risk-composite}.
\end{proof}

The converse is generally false: equality of risk does not imply equality of predictors. Thus scalar risk is coarser than the induced predictor.

\subsection{Representation properties and identifiability}
\label{subsec:representation-properties}

A representation property is a predicate or functional whose value is assigned to the representation component of a factorized predictor. Since representation spaces may vary across admissible pairs, the representation object includes both the measurable map and its codomain. This convention allows properties such as dimensionality, compression, or augmentation by additional coordinates to be treated explicitly.

\begin{definition}[Representation property]
\label{def:representation-property}
A Boolean-valued representation property on $\mathcal M$ is a map
\begin{equation}
\mathcal P:\mathcal M\to\{0,1\}
\end{equation}
whose value depends on the representation component alone. That is, if $(h,c),(h,c')\in\mathcal M$ have the same representation object $h$, then
\begin{equation}
\mathcal P(h,c)=\mathcal P(h,c').
\end{equation}
When no ambiguity arises, the $\mathcal P(h,c)$ is written as $\mathcal P(h)$.
\end{definition}

Examples include minimality relative to a target statistic, compression, invariance, equivariance, nuisance invariance, semantic alignment, and disentanglement \citep{fisher1922mathematical, tishby1999information, bengio2013representation, cohen2016group, locatello2019challenging,scholkopf2021toward}. Some such properties depend on auxiliary structure, such as a group action on
the input or representation space, a semantic reference structure, or a class of allowed coordinate transformations \citep{cohen2016group, higgins2017beta, locatello2019challenging, scholkopf2021toward}. Formally, such a property may be written as $\mathcal P(h;\mathcal S)$, where the auxiliary structure $\mathcal S$ is held fixed throughout the identifiability comparison. Equivalently, $\mathcal S$ may be regarded as part of the admissible model class $\mathcal M$. Although Definition~\ref{def:representation-property} is stated for Boolean-valued properties, the definitions below extend verbatim to properties valued in any set.

Two levels of identifiability are distinguished. The first asks whether a representation property is determined by the induced predictor $c\circ h$.

\begin{definition}[Identifiability from the induced predictor]
\label{def:prediction-identifiability}
A representation property $\mathcal P$ is identifiable from the induced predictor within $\mathcal M$ if, for all $(h,c),(\tilde h,\tilde c)\in\mathcal M$,
\begin{equation}
(h,c)\sim_P(\tilde h,\tilde c)
\quad\Longrightarrow\quad
\mathcal P(h)=\mathcal P(\tilde h).
\end{equation}
\end{definition}
For brevity, such a property is called \emph{predictor-identifiable}.

The second asks whether a representation property is determined by the scalar value of the supervised risk.

\begin{definition}[Identifiability from supervised risk]
\label{def:risk-identifiability}
A representation property $\mathcal P$ is identifiable from supervised risk within $\mathcal M$ if, for all $(h,c),(\tilde h,\tilde c)\in\mathcal M$,
\begin{equation}
R_P(c\circ h)=R_P(\tilde c\circ\tilde h)
\quad\Longrightarrow\quad
\mathcal P(h)=\mathcal P(\tilde h).
\end{equation}
\end{definition}
A property identifiable from the scalar supervised risk within $\mathcal M$ is
called \emph{risk-identifiable}.

This definition concerns only the scalar population-risk value. It does not assume access to the induced predictor, conditional risks, loss distributions, or an algorithmic selection rule. Because predictor equivalence implies equality of risk, identifiability from scalar risk imposes a stronger invariance requirement than identifiability from the induced predictor.

\begin{proposition}[Risk-identifiability implies induced-predictor identifiability]
\label{prop:risk-identifiability-stronger}
If $\mathcal P$ is identifiable from supervised risk within $\mathcal M$, then $\mathcal P$ is identifiable from the induced predictor within $\mathcal M$.
\end{proposition}

\begin{proof}
Assume $\mathcal P$ is identifiable from supervised risk. Let $(h,c),(\tilde h,\tilde c)\in\mathcal M$ satisfy
\begin{equation}
(h,c)\sim_P(\tilde h,\tilde c).
\end{equation}
By Corollary~\ref{cor:prediction-implies-risk},
\begin{equation}
R_P(c\circ h)=R_P(\tilde c\circ\tilde h).
\end{equation}
Risk-identifiability then gives
\begin{equation}
\mathcal P(h)=\mathcal P(\tilde h).
\end{equation}
Hence $\mathcal P$ is identifiable from the induced predictor.
\end{proof}

Equivalently, any representation property that is not identifiable from the induced predictor cannot be identifiable from scalar supervised risk.

\subsection{The fiber/descent criterion}
\label{subsec:fiber-descent}

The following descent characterization is the central formal criterion. It is an
instance of the general principle that an object is identifiable from an
observation map exactly when it is constant on the observational equivalence
classes induced by that map \citep{paulino1994identifiability,
basse2020general}. It expresses induced-predictor identifiability as constancy over all admissible factorizations of the same predictor.

\begin{theorem}[Fiber/descent criterion]
\label{thm:fiber-descent}
Let $\mathcal M$ be a class of admissible representation--head pairs, and let $\mathcal P:\mathcal M\to\{0,1\}$ be a representation property. The following are equivalent.

\begin{enumerate}
\item $\mathcal P$ is identifiable from the induced predictor within $\mathcal M$.

\item $\mathcal P$ is constant on the fibers of
\begin{equation}
\Pi:\mathcal M\to\Pi(\mathcal M),
\qquad
\Pi(h,c)=[c\circ h]_{P_X}.
\end{equation}
Equivalently, whenever
\begin{equation}
c(h(X))=\tilde c(\tilde h(X))
\qquad
P_X\text{-a.s.},
\end{equation}
one has
\begin{equation}
\mathcal P(h)=\mathcal P(\tilde h).
\end{equation}

\item There exists a map
\begin{equation}
\overline{\mathcal P}:\Pi(\mathcal M)\to\{0,1\}
\end{equation}
such that, for every $(h,c)\in\mathcal M$,
\begin{equation}
\mathcal P(h)=\overline{\mathcal P}([c\circ h]_{P_X}).
\end{equation}
\end{enumerate}
\end{theorem}

\begin{proof}
The equivalence of $(1)$ and $(2)$ follows directly from Definition~\ref{def:prediction-identifiability}, Definition~\ref{def:predictor-equivalence}, and the fact that the fibers of $\Pi$ are the predictor-equivalence classes.

For $(2)\Rightarrow(3)$, let $[f]_{P_X}\in\Pi(\mathcal M)$. Choose any pair $(h,c)\in\mathcal M$ such that
\begin{equation}
[c\circ h]_{P_X}=[f]_{P_X}.
\end{equation}
Define
\begin{equation}
\overline{\mathcal P}([f]_{P_X})=\mathcal P(h).
\end{equation}
This is well-defined because if another pair $(\tilde h,\tilde c)\in\mathcal M$ also satisfies
\begin{equation}
[\tilde c\circ\tilde h]_{P_X}=[f]_{P_X},
\end{equation}
then $(h,c)$ and $(\tilde h,\tilde c)$ lie in the same fiber of $\Pi$. By $(2)$,
\begin{equation}
\mathcal P(h)=\mathcal P(\tilde h).
\end{equation}
Thus $\overline{\mathcal P}$ is independent of the chosen factorization, and by construction
\begin{equation}
\mathcal P(h)=\overline{\mathcal P}([c\circ h]_{P_X}).
\end{equation}

For $(3)\Rightarrow(2)$, suppose there exists $\overline{\mathcal P}$ such that
\begin{equation}
\mathcal P(h)=\overline{\mathcal P}([c\circ h]_{P_X})
\end{equation}
for all $(h,c)\in\mathcal M$. If
\begin{equation}
[c\circ h]_{P_X}=[\tilde c\circ\tilde h]_{P_X},
\end{equation}
then
\begin{equation}
\mathcal P(h)
=
\overline{\mathcal P}([c\circ h]_{P_X})
=
\overline{\mathcal P}([\tilde c\circ\tilde h]_{P_X})
=
\mathcal P(\tilde h).
\end{equation}
Hence $\mathcal P$ is constant on the fibers of $\Pi$.
\end{proof}

The criterion says that a representation property is identifiable from the induced predictor exactly when it descends to a well-defined property of that predictor. In the language of Figure~\ref{fig:predictor-equivalence-fibers}, $\mathcal P$ must assign the same value to every point in a given predictor fiber. If a property changes within a fiber, then observing the induced predictor cannot determine it.

\begin{example}[A positive case: injective heads make invariance descend]
\label{ex:positive-injective-invariance}
The following example illustrates that additional structural assumptions can make
a representation-level property descend to a predictor-level property. This is
analogous to the way architectural or structural constraints can enforce
invariance or equivariance in representation learning
\citep{cohen2016group, bronstein2021geometric}. Let a fixed measurable transformation $g:\mathcal X\to\mathcal X$ be given, and consider an admissible class $\mathcal M_{\mathrm{inj}}$ in which, for every $(h,c)\in\mathcal M_{\mathrm{inj}}$, the head $c$ is injective on all representation values attained by $h(X)$ and $h(g(X))$, up to $P_X$-null sets. Define
\begin{equation}
\mathcal P_g(h)
=
\mathbf 1\{h(g(X))=h(X)\; P_X\text{-a.s.}\}.
\end{equation}
For any $(h,c)\in\mathcal M_{\mathrm{inj}}$ with induced predictor $f=c\circ h$,
\begin{equation}
h(g(X))=h(X)
\quad\Longleftrightarrow\quad
f(g(X))=f(X)
\qquad
P_X\text{-a.s.},
\end{equation}
where the reverse implication uses injectivity of $c$ on the relevant representation values. Thus
\begin{equation}
\mathcal P_g(h)=\overline{\mathcal P}_g([c\circ h]_{P_X}),
\qquad
\overline{\mathcal P}_g([f]_{P_X})
=
\mathbf 1\{f(g(X))=f(X)\; P_X\text{-a.s.}\}.
\end{equation}
Within this restricted class, representation invariance to $g$ is identifiable from the induced predictor.
\end{example}

\begin{corollary}[Variation within a fiber implies non-identifiability]
\label{cor:variation-fiber}
Let $\mathcal M$ be a class of admissible representation--head pairs, and let $\mathcal P$ be a representation property. Suppose there exist $(h,c),(\tilde h,\tilde c)\in\mathcal M$ such that
\begin{equation}
[c\circ h]_{P_X}=[\tilde c\circ\tilde h]_{P_X},
\end{equation}
but
\begin{equation}
\mathcal P(h)\neq\mathcal P(\tilde h).
\end{equation}
Then $\mathcal P$ is not predictor-identifiable within $\mathcal M$. It is therefore not risk-identifiable within $\mathcal M$.
\end{corollary}

\begin{proof}
The property $\mathcal P$ varies within a fiber of $\Pi$, so it fails the fiber-constancy condition in Theorem~\ref{thm:fiber-descent}. Hence $\mathcal P$ is not predictor-identifiable. If it were risk-identifiable, Proposition~\ref{prop:risk-identifiability-stronger} would imply predictor-identifiability, a contradiction.
\end{proof}

\subsection{Predictor-preserving augmentation}
\label{subsec:predictor-preserving-augmentation}

A canonical mechanism for generating variation within a predictor fiber is predictor-preserving augmentation. A representation can be augmented by admissible side information while the prediction head is extended so that it ignores the added coordinate. Figure~\ref{fig:predictor-preserving-augmentation} illustrates the construction: the representation is enlarged, but the induced predictor is unchanged.

\begin{figure}[t]
\centering
\includegraphics[width=0.92\linewidth]{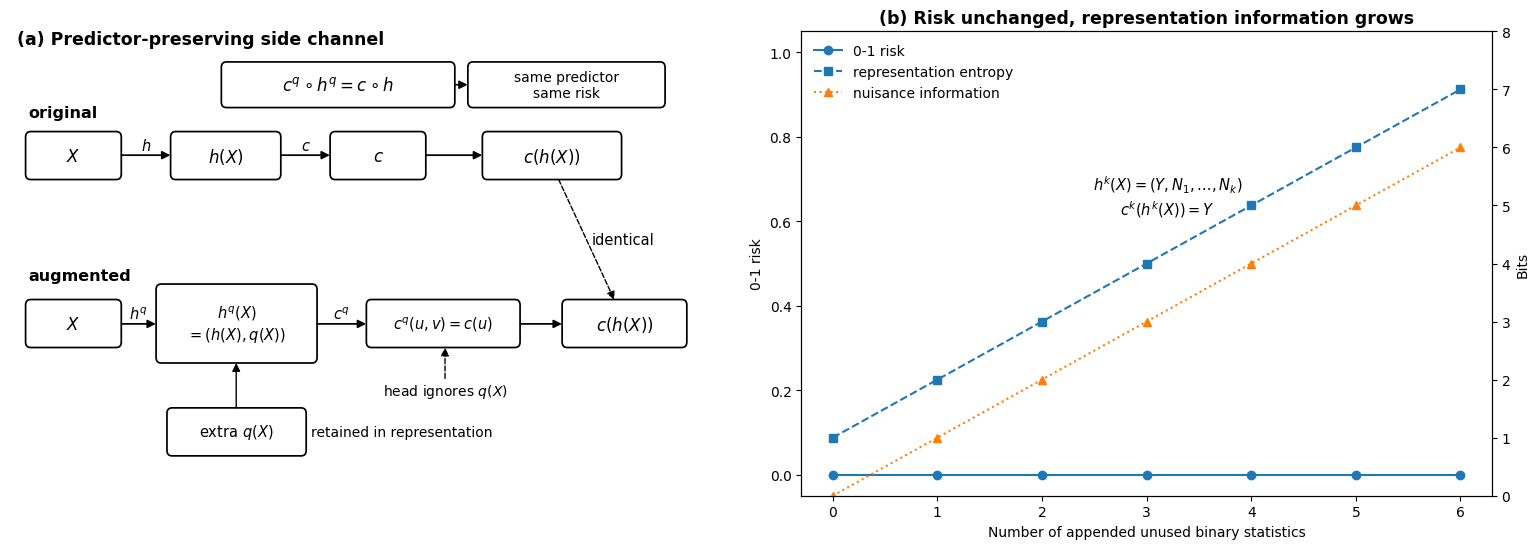}
\caption{Predictor-preserving augmentation. The original predictor is $c\circ h$. After augmentation, the representation becomes $h^q(x)=(h(x),q(x))$, while the augmented head $c^q(u,v)=c(u)$ ignores the added coordinate. Hence $c^q\circ h^q=c\circ h$ pointwise. The representation can retain additional information even though the supervised predictor and supervised risk are unchanged.}
\label{fig:predictor-preserving-augmentation}
\end{figure}

\begin{proposition}[Predictor-preserving augmentation]
\label{prop:predictor-preserving-augmentation}
Let $h:\mathcal X\to\mathcal H$ be a measurable representation and let $c:\mathcal H\to\mathcal A$ be a measurable head. Let $q:\mathcal X\to\mathcal Q$ be a measurable statistic, where $(\mathcal Q,\mathcal A_{\mathcal Q})$ is a measurable space. Define
\begin{equation}
h^q(x)=(h(x),q(x))
\end{equation}
and
\begin{equation}
c^q(u,v)=c(u).
\end{equation}
Then
\begin{equation}
c^q\circ h^q=c\circ h
\end{equation}
pointwise. Consequently,
\begin{equation}
R_P(c^q\circ h^q)=R_P(c\circ h).
\end{equation}
\end{proposition}

Here $q$ is a measurable statistic on the chosen input space. In empirical settings, the auxiliary coordinate may instead be an annotation, metadata field, or transformation identifier, provided it is treated as admissible side information. Under a raw-observation convention, this corresponds either to working on an enlarged observation space or to considering model classes that can accept the auxiliary coordinate.

\begin{proof}
For every $x\in\mathcal X$,
\begin{equation}
(c^q\circ h^q)(x)
=
c^q(h(x),q(x))
=
c(h(x))
=
(c\circ h)(x).
\end{equation}
Thus the two composite predictors are identical pointwise. Equality of risks follows immediately.
\end{proof}

Proposition~\ref{prop:predictor-preserving-augmentation} shows that predictor-preserving augmentation produces another admissible factorization in the same predictor fiber. The construction therefore concerns non-identifiability at the level of admissible factorizations. 

\subsection{The augmentation obstruction}
\label{subsec:augmentation-obstruction}

The augmentation obstruction applies to model classes that are rich enough to contain predictor-preserving augmentations.

\begin{definition}[Closure under predictor-preserving augmentation]
\label{def:augmentation-closure}
A model class $\mathcal M$ is closed under predictor-preserving augmentation if, for every $(h,c)\in\mathcal M$ and every admissible measurable statistic $q:\mathcal X\to\mathcal Q$, the augmented pair
\begin{equation}
h^q(x)=(h(x),q(x)),
\qquad
c^q(u,v)=c(u)
\end{equation}
also belongs to $\mathcal M$.
\end{definition}

The qualifier ``admissible'' permits $\mathcal M$ to restrict the allowable auxiliary statistics, representation spaces, or augmented heads. For an unrestricted measurable model class, every measurable statistic $q$ is admissible. In restricted hypothesis classes, the obstruction applies whenever the class contains predictor-equivalent factorizations on which the representation property takes different values.

\begin{definition}[Augmentation-sensitive property]
\label{def:augmentation-sensitive}
Let $\mathcal M$ be closed under predictor-preserving augmentation. A representation property $\mathcal P$ is augmentation-sensitive within $\mathcal M$ if there exist $(h,c)\in\mathcal M$ and an admissible statistic $q:\mathcal X\to\mathcal Q$ such that
\begin{equation}
\mathcal P(h)\neq\mathcal P(h^q),
\end{equation}
where
\begin{equation}
h^q(x)=(h(x),q(x)).
\end{equation}
\end{definition}

\begin{theorem}[Augmentation obstruction]
\label{thm:augmentation-obstruction}
Let $\mathcal M$ be a class of admissible representation--head pairs closed under predictor-preserving augmentation. If a representation property $\mathcal P$ is augmentation-sensitive within $\mathcal M$, then $\mathcal P$ is not predictor-identifiable within $\mathcal M$. In particular, it is not risk-identifiable within $\mathcal M$.
\end{theorem}

\begin{proof}
Since $\mathcal P$ is augmentation-sensitive, there exist $(h,c)\in\mathcal M$ and an admissible statistic $q:\mathcal X\to\mathcal Q$ such that
\begin{equation}
\mathcal P(h)\neq \mathcal P(h^q).
\end{equation}
By closure under predictor-preserving augmentation, the augmented pair $(h^q,c^q)$ belongs to $\mathcal M$, where
\begin{equation}
h^q(x)=(h(x),q(x)),
\qquad
c^q(u,v)=c(u).
\end{equation}
By Proposition~\ref{prop:predictor-preserving-augmentation},
\begin{equation}
c^q\circ h^q=c\circ h.
\end{equation}
Thus $(h,c)$ and $(h^q,c^q)$ lie in the same predictor-equivalence class. However, $\mathcal P$ takes different values on their representations. By Corollary~\ref{cor:variation-fiber}, $\mathcal P$ is not predictor-identifiable and is not risk-identifiable.
\end{proof}

Closure under predictor-preserving augmentation is only a sufficient condition. The underlying obstruction is variation of $\mathcal P$ within a predictor fiber; augmentation closure is one general mechanism that guarantees such variation.

Theorem~\ref{thm:augmentation-obstruction} provides a broad witness for non-identifiability: one may append unused admissible information to the representation and let the head ignore it. The next section applies this obstruction to common representation-level desiderata.

\subsection{Restoring identifiability by restricting the admissible fiber}
\label{subsec:restoring-identifiability}

The preceding obstruction is conditional on the admissible class. A
representation property fails to be identifiable whenever the admissible class
contains predictor-equivalent factorizations on which the property differs.
Restoring identifiability therefore amounts to changing the relevant
equivalence classes, either by restricting the model class, adding assumptions,
or enriching the observational object \citep{paulino1994identifiability,
basse2020general, locatello2019challenging}. This can be done by restricting the admissible class,
selecting a canonical representative inside each predictor fiber, or enriching
the observations so that factorizations that were previously equivalent become
distinguishable.

Let $\mathcal M_0\subseteq\mathcal M$ be a restricted admissible class. The
projection relevant to this restricted problem is
\[
\Pi_0:\mathcal M_0\to \Pi_0(\mathcal M_0),
\qquad
\Pi_0(h,c)=[c\circ h]_{P_X}.
\]
A representation property $\mathcal P$ that fails to be identifiable on
$\mathcal M$ can become identifiable on $\mathcal M_0$ when the restriction
removes the predictor-preserving alternatives on which $\mathcal P$ varies.
The restoration condition is the fiber condition applied to the restricted
fibers of $\Pi_0$.

\begin{proposition}[Restoration by fiber restriction]
\label{prop:restoration-fiber-restriction}
Let $\mathcal M_0\subseteq\mathcal M$ be a restricted admissible class, and let
$\mathcal P:\mathcal M_0\to\{0,1\}$ be a representation property. Then
$\mathcal P$ is identifiable from the induced predictor within $\mathcal M_0$
if and only if
\[
\Pi_0(h,c)=\Pi_0(\tilde h,\tilde c)
\quad\Longrightarrow\quad
\mathcal P(h)=\mathcal P(\tilde h)
\]
for all $(h,c),(\tilde h,\tilde c)\in\mathcal M_0$.
Equivalently, $\mathcal P$ is identifiable within $\mathcal M_0$ exactly when
it is constant on the restricted fibers of $\Pi_0$.
\end{proposition}

\begin{proof}
This is Theorem~\ref{thm:fiber-descent} applied to the restricted admissible
class $\mathcal M_0$.
\end{proof}

Restoration is the descent criterion applied after the admissible alternatives
have changed. In a broad class, predictor-preserving augmentation can append
unused coordinates and thereby change compression, invariance, nuisance
information, or semantic accessibility without changing the predictor. In a
restricted class, such augmentations may be excluded. Fixed-dimensional bottlenecks, equivariant architectures, injective heads,
minimality constraints, and explicit information constraints all shrink the
predictor fiber \citep{tishby1999information, alemi2017deep, cohen2016group,
bronstein2021geometric}. Once the remaining factorizations of a given predictor agree on the
representation property of interest, the property descends to the predictor
within that restricted class.

Assumptions restore identifiability by enforcing this fiber constancy. For
example, invariance of a representation to a transformation $g$ fails to
descend in a class that permits the augmentation
\[
h^q(x)=(h(x),q(x)),
\qquad
c^q(u,v)=c(u),
\]
with $q(g\cdot X)\neq q(X)$ with positive probability. The augmented
representation is transformation-sensitive while the predictor is unchanged.
By contrast, if the admissible class requires the head to be injective on the
relevant representation values, then representation invariance and predictor
invariance coincide:
\[
h(g\cdot X)=h(X)
\quad\Longleftrightarrow\quad
c(h(g\cdot X))=c(h(X))
\qquad P_X\text{-a.s.}
\]
Within this restricted class, the representation-level invariance property
descends to the predictor-level property
\[
f(g\cdot X)=f(X)
\qquad P_X\text{-a.s.},
\]
and is therefore identifiable from the induced predictor. The same pattern
applies to other properties: minimality is identifiable within a class that
admits only minimal sufficient representations; equivariance is identifiable
within a class whose representation spaces and heads enforce the specified
group action; and nuisance removal is identifiable once admissible alternatives
that retain unused nuisance coordinates have been excluded or directly
measured.

Identifiability can also be restored by supplementing the predictor with a
selection rule. Suppose a deterministic rule selects one admissible
factorization for each induced predictor,
\[
S:\Pi(\mathcal M)\to\mathcal M,
\qquad
\Pi(S([f]_{P_X}))=[f]_{P_X}.
\]
Then
\[
\mathcal P_S([f]_{P_X})
=
\mathcal P\big(S([f]_{P_X})\big)
\]
is a well-defined property of the predictor together with the selection rule.
This formalizes the role of optimization bias, initialization, regularization,
early stopping, and model selection as mechanisms that select representatives
from otherwise indistinguishable predictor fibers
\citep{hardt2016train, neyshabur2017exploring, soudry2018implicit, sevetlidis2026trainingmemory}. Such mechanisms choose representatives
from predictor fibers, so the identified object is the property of the
selected representative.

A third route is to enlarge the observational object. The non-identifiability
results above use the induced supervised predictor as the observable. If representation activations, probes, transformation responses, auxiliary labels, environments, interventions, or causal measurements are also observed,
then the observational equivalence classes become finer than the fibers of
$\Pi$ \citep{alain2017understanding, arjovsky2019invariant,peters2016causal, scholkopf2021toward}. Systems that are equivalent as supervised predictors may become
distinguishable once these additional measurements are included.

The positive and negative statements therefore have the same form. A
representation-level property is identifiable exactly when it is constant on
the relevant observational equivalence class. Supervised prediction alone uses
the fibers of $\Pi$. Architectural restrictions, objectives, selection rules,
auxiliary measurements, multiple environments, and causal assumptions replace
those fibers by smaller equivalence classes or by selected representatives.
The augmentation obstruction identifies the fiber variation responsible for
failure in broad classes; restoration occurs when additional structure removes
that variation.

\subsection{Consequences for Common Representation Properties}
\label{sec:examples}

After the obstruction and restoration mechanisms have been characterized at
the fiber level, the augmentation obstruction can be applied to common
representation-level desiderata. This section applies the augmentation obstruction to common representation-level properties, identifying conditions under which each property varies inside a predictor fiber.

Table~\ref{tab:augmentation-sensitive-properties} summarizes several such
properties. Each row should be read conditionally: it gives a common way in
which the property can change under predictor-preserving augmentation, provided
that the auxiliary statistic and the augmented representation belong to the
admissible model class.

\begin{table}[t]
\centering
\caption{Common augmentation-sensitive representation properties. Each row
describes auxiliary information that can change the representation property
while leaving the induced predictor unchanged, as in
Theorem~\ref{thm:augmentation-obstruction}.}
\label{tab:augmentation-sensitive-properties}
\small
\begin{tabular}{L{0.22\linewidth}L{0.38\linewidth}L{0.30\linewidth}}
\toprule
Property & How predictor-preserving augmentation can change it & Extra structure required \\
\midrule
Minimality relative to $B$ &
Append $q(X)$ with $\sigma(q(X))\not\subseteq\sigma(B(X))$ &
Target statistic $B$ \\

Compression &
Append high-information, high-dimensional, or high-rank information &
Entropy, dimension, code-length, rank, or information criterion \\

Nuisance invariance &
Append a statistic dependent on nuisance $N$ conditional on task $T$ &
Task and nuisance variables $(T,N)$ \\

Transformation invariance &
Append $q$ such that $q(g\cdot X)\neq q(X)$ with positive probability &
Group action on inputs \\

Equivariance &
Append a coordinate incompatible with the chosen representation-space action &
Group actions on input and representation spaces \\

Semantic accessibility &
Append a semantic attribute, annotation, or metadata field &
Semantic or generative reference structure \\
\bottomrule
\end{tabular}
\end{table}

The remainder of the section spells out two representative cases. The other
rows follow by the same fiber-variation argument.

\paragraph{Minimality.}
Let $B:\mathcal X\to\mathcal B$ be a measurable statistic regarded as
target-relevant for the supervised problem, in the spirit of sufficient or
minimal sufficient summaries in statistics and sufficient dimension reduction
\citep{fisher1922mathematical, lehmann1998theory, cook1998regression}. For example, $B(X)$ may represent a
Bayes-optimal decision, a conditional law, a semantic task variable, or another
externally specified statistic. Say that a representation $h$ is minimal
relative to $B$ when
\[
\sigma(h(X))=\sigma(B(X))
\qquad
\text{mod }P_X.
\]
If $h$ is minimal relative to $B$, but an admissible statistic $q$ satisfies
\[
\sigma(q(X))\not\subseteq \sigma(B(X))
\qquad
\text{mod }P_X,
\]
then the augmented representation $h^q$ is no longer minimal relative to $B$.
The reason is that $h^q$ contains information not contained in
$\sigma(B(X))$. Since the predictor is nevertheless preserved by
Proposition~\ref{prop:predictor-preserving-augmentation}, minimality relative
to $B$ varies within a predictor fiber. By
Theorem~\ref{thm:augmentation-obstruction}, it is therefore not
predictor-identifiable in any admissible class containing both factorizations.

This separates predictive sufficiency from representation minimality.
Supervised prediction may show that a representation supports accurate
prediction, but it does not by itself certify that unused information has been
excluded.

\paragraph{Transformation invariance.}
Let $G$ be a group acting measurably on $\mathcal X$, and fix a transformation
$g\in G$. Group actions provide the standard formal language for invariance and
equivariance in modern representation learning
\citep{cohen2016group, bronstein2021geometric}. Consider the representation-level invariance property
\[
h(g\cdot X)=h(X)
\qquad
P_X\text{-a.s.}
\]
If this property holds for $h$, but an admissible statistic $q$ satisfies
\[
q(g\cdot X)\neq q(X)
\]
with positive probability, then the augmented representation $h^q$ is not
invariant to $g$. Indeed, its added coordinate changes under the transformation.
The induced predictor is still preserved by
Proposition~\ref{prop:predictor-preserving-augmentation}. Hence representation
invariance to $g$ varies within a predictor fiber, and
Theorem~\ref{thm:augmentation-obstruction} implies that it is not
predictor-identifiable whenever both the invariant representation and its
predictor-preserving non-invariant augmentation are admissible.

The same reasoning applies to equivariance. Since equivariance is defined
relative to specified group actions on both the input and representation
spaces, an appended coordinate with an incompatible or transformation-sensitive
action can break equivariance without changing the supervised predictor.

The remaining properties in
Table~\ref{tab:augmentation-sensitive-properties} are instances of the same
criterion. Nuisance information, compression measures, and semantic accessibility are not
determined by the induced predictor whenever they can be altered by admissible
predictor-preserving augmentation. This complements existing work emphasizing
that shortcut, nuisance, disentangled, or causal structure requires additional
assumptions, measurements, or inductive biases beyond ordinary predictive
performance
\citep{geirhos2020shortcut, locatello2019challenging, arjovsky2019invariant,
scholkopf2021toward}.

\section{Experiments}
\label{sec:experiments}

The preceding sections establish an exact predictor-fiber obstruction. The experiments connect this obstruction to standard neural-network settings by separating predictor-level evidence from representation-level measurements.

Two complementary diagnostic studies are reported. The first constructs exact empirical predictor fibers and measures how standard representation diagnostics vary within them.  Because the augmented head ignores the appended coordinate, predictor preservation is guaranteed before any empirical measurement is made. The role of the empirical tables is therefore to document which familiar diagnostics---probe accuracy, invariance distance, effective rank, and domain decodability---can change inside an exactly fixed predictor fiber. The second study is different. It does not hold the predictor fixed pointwise. Instead, it compares Waterbirds models trained with different representation-level constraints after matching supervised performance. This study illustrates how common objectives can select different representation diagnostics among models with similar task behavior. Same-seed near-fiber results are reported separately in Appendix~\ref{app:near-fiber-details}.

\subsection{Algebraic witnesses: diagnostics can vary while the predictor is fixed}
\label{subsec:experiments-exact-witnesses}

The exact witnesses instantiate the construction
\[
h^q(x)=(h(x),q(x)),
\qquad
c^q(u,v)=c(u).
\]
Therefore,
\[
c^q(h^q(x))=c(h(x))
\]
for every evaluated input $x$. The equality is algebraic, not statistical. Consequently, zero predictor disagreement and zero logit difference are implementation checks rather than empirical findings.

The appended coordinate $q$ is chosen precisely because it changes a representation diagnostic: a semantic attribute changes semantic accessibility, a transformation-sensitive code changes transformation decodability or invariance distance, and a domain label changes domain decodability. These examples are intentionally direct. Their purpose is to exhibit finite-sample pairs of representations in the same predictor fiber on which common diagnostics take different values. These witnesses use an admissible observation space that includes annotations, metadata, or transformation identifiers. Under a raw-image-only convention, the same construction applies to the corresponding enlarged observation space.

\begin{table}[t]
\centering
\caption{Exact finite-sample witnesses of fiber variation. Predictor equality is guaranteed by the construction $c^q(u,v)=c(u)$; the zero-disagreement column is therefore an implementation check. The diagnostic column records representation-level quantities that change while the induced predictor is fixed pointwise.}
\label{tab:exact-witness-summary}
\scriptsize
\setlength{\tabcolsep}{3pt}
\begin{tabular}{L{0.13\linewidth}L{0.18\linewidth}L{0.18\linewidth}L{0.38\linewidth}}
\toprule
Dataset & Auxiliary statistic $q$ & Fiber certificate & Representation diagnostic changed \\
\midrule
CelebA & Non-target semantic attributes & $0$ disagreement; $0$ max logit difference & Attribute-probe AUC becomes $1.000$ for all $39$ appended attributes; mean $\Delta$AUC $0.137$. \\
CIFAR-10 & Learned transformation embedding & Accuracy and CE unchanged; $0$ disagreement & Transformation-probe accuracy changes from $0.5948$ to $0.7907$; invariance distances increase for all reported transformations. \\
STL-10 & Learned transformation embedding & Accuracy and CE unchanged; $0$ disagreement & Transformation-probe accuracy changes from $0.5732$ to $0.7034$; invariance distances increase for all reported transformations. \\
OfficeHome & One-hot domain label & Class accuracy and loss unchanged; $0$ disagreement & Domain-probe accuracy changes from $0.667\pm0.009$ to $1.000\pm0.000$. \\
PACS & One-hot domain label & Class accuracy and loss unchanged; $0$ disagreement & Domain-probe accuracy changes from $0.949\pm0.002$ to $1.000\pm0.000$. \\
\bottomrule
\end{tabular}
\end{table}

Because $c^q(h^q(x))=c(h(x))$ pointwise, the zero-disagreement column certifies that the compared systems lie in the same empirical predictor fiber. The remaining columns record deliberately induced changes in representation diagnostics: semantic attribute accessibility in CelebA, transformation decodability and transformation sensitivity in CIFAR-10 and STL-10, and domain decodability in OfficeHome and PACS.

Formally, $q$ is a measurable statistic on the chosen input space. Empirically, the appended coordinate may come from annotations, metadata, or transformation identifiers allowed as side information for the diagnostic. Under a strict raw-pixel input convention, the witness concerns either an enlarged observation space or a model class that can receive such side information.

\subsection{Constraint-selected representations at matched supervised performance}
\label{subsec:experiments-constraints}

The Waterbirds study examines a matched-performance setting. The supervised label is bird type, waterbird versus landbird, and the nuisance variable is background, water versus land. This dataset is useful because the nuisance variable is semantically meaningful and directly probeable. All models use an ImageNet-pretrained ResNet-18 backbone and a linear task head. Ten seeds are trained for each method.

Besides ERM, the comparison includes a deterministic bottleneck, a variational information bottleneck (VIB), an explicit augmentation-invariance penalty, and a supervised contrastive objective. For each seed and method, model selection matches the corresponding ERM run by choosing the configuration closest in validation accuracy and cross-entropy. The adversarial nuisance-invariance variant is reported in Appendix~\ref{app:waterbirds-additional}, since in this run it changed representation geometry but did not consistently reduce background decodability.

\begin{table}[t]
\centering
\caption{Matched supervised performance on Waterbirds. Values are mean
$\pm$ standard deviation over ten seeds.}
\label{tab:waterbirds-task}
\small
\begin{tabular}{lccc}
\toprule
Method & Test acc. & CE loss & Worst-group acc. \\
\midrule
ERM & $87.7 \pm 1.4$ & $0.375 \pm 0.047$ & $59.0 \pm 11.7$ \\
Bottleneck & $88.3 \pm 1.2$ & $0.339 \pm 0.039$ & $55.2 \pm 8.4$ \\
VIB & $88.8 \pm 1.0$ & $0.298 \pm 0.019$ & $54.5 \pm 7.8$ \\
AugInv & $89.0 \pm 1.3$ & $0.340 \pm 0.038$ & $63.1 \pm 11.9$ \\
SupCon & $88.8 \pm 1.8$ & $0.334 \pm 0.041$ & $60.0 \pm 7.6$ \\
\bottomrule
\end{tabular}
\end{table}

Table~\ref{tab:waterbirds-task} shows that the selected models have comparable task-level performance. The constrained models have test accuracies between $88.3\%$ and $89.0\%$, compared with $87.7\%$ for ERM. Cross-entropy is also comparable, and in several constrained models lower than ERM. Thus the representation differences below are not explained by a collapse in supervised performance.

\begin{table}[t]
\centering
\caption{Representation diagnostics on Waterbirds after supervised-performance
matching. Nuisance probe is balanced accuracy for predicting the background
from frozen representations.}
\label{tab:waterbirds-repr}
\small
\begin{tabular}{lccc}
\toprule
Method & Nuisance probe & Invariance dist. & Effective rank \\
\midrule
ERM & $88.7 \pm 0.5$ & $0.387 \pm 0.098$ & $35.7 \pm 18.9$ \\
Bottleneck & $82.1 \pm 2.7$ & $0.278 \pm 0.096$ & $1.9 \pm 0.4$ \\
VIB & $85.5 \pm 1.2$ & $0.444 \pm 0.133$ & $1.8 \pm 1.2$ \\
AugInv & $89.3 \pm 0.5$ & $0.326 \pm 0.098$ & $20.8 \pm 14.8$ \\
SupCon & $89.3 \pm 0.6$ & $0.361 \pm 0.089$ & $18.7 \pm 6.9$ \\
\bottomrule
\end{tabular}
\end{table}

The matched models differ sharply at the representation level, as shown in Table~\ref{tab:waterbirds-repr}. The deterministic bottleneck and VIB objectives select extremely low-rank representations: their effective ranks are approximately $1.9$ and $1.8$, compared with $35.7$ for ERM. These objectives also reduce linear decodability of the background variable. By contrast, the augmentation-invariance objective primarily reduces transformation distance while preserving high background decodability. Supervised contrastive training produces an intermediate representation: lower effective rank and mildly lower transformation distance than ERM, but no reduction in nuisance probe accuracy.

\begin{figure}[t]
\centering
\begin{minipage}{0.49\linewidth}
\centering
\includegraphics[width=\linewidth]{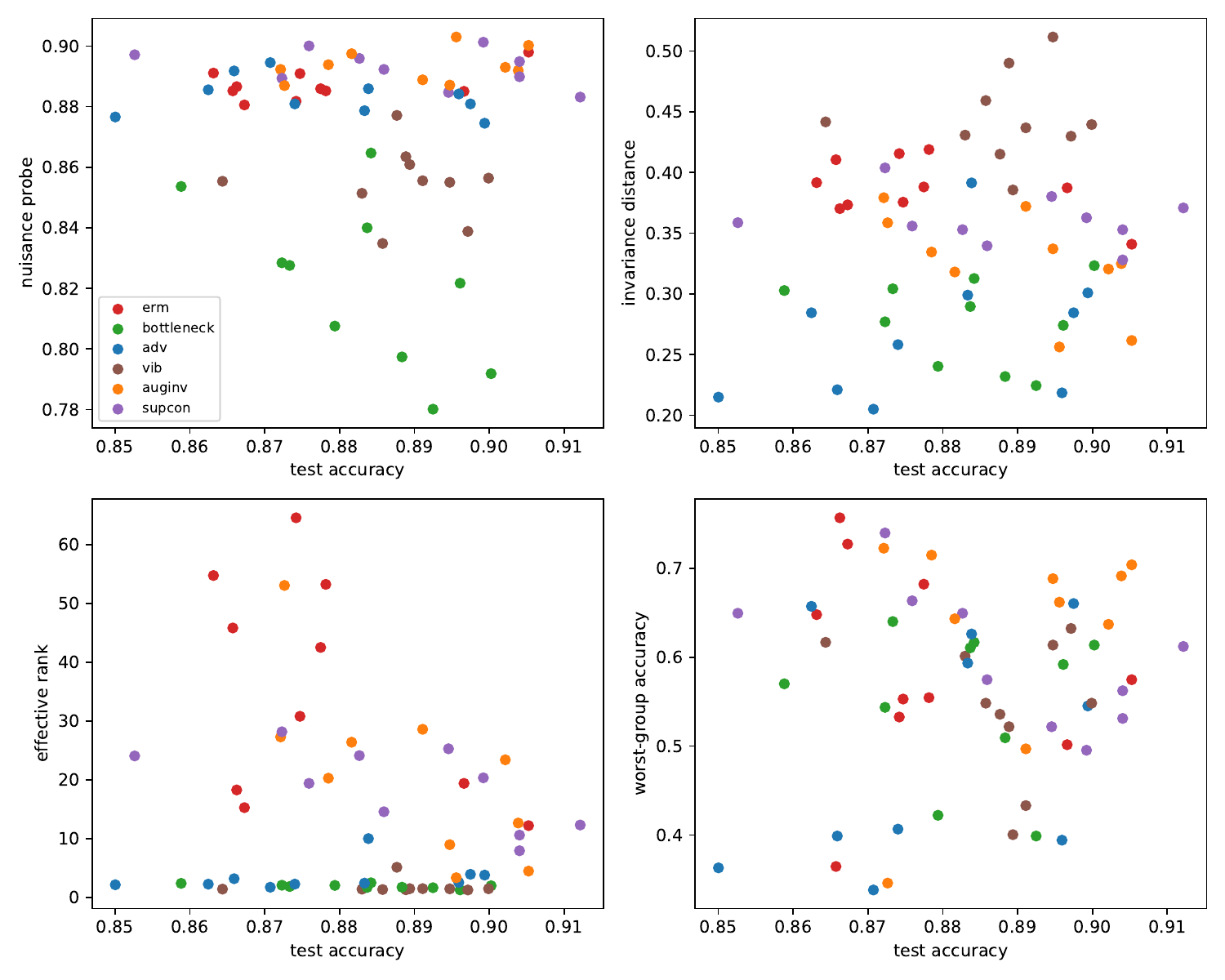}\\
{\small (a) Accuracy and diagnostics}
\end{minipage}
\hfill
\begin{minipage}{0.49\linewidth}
\centering
\includegraphics[width=\linewidth]{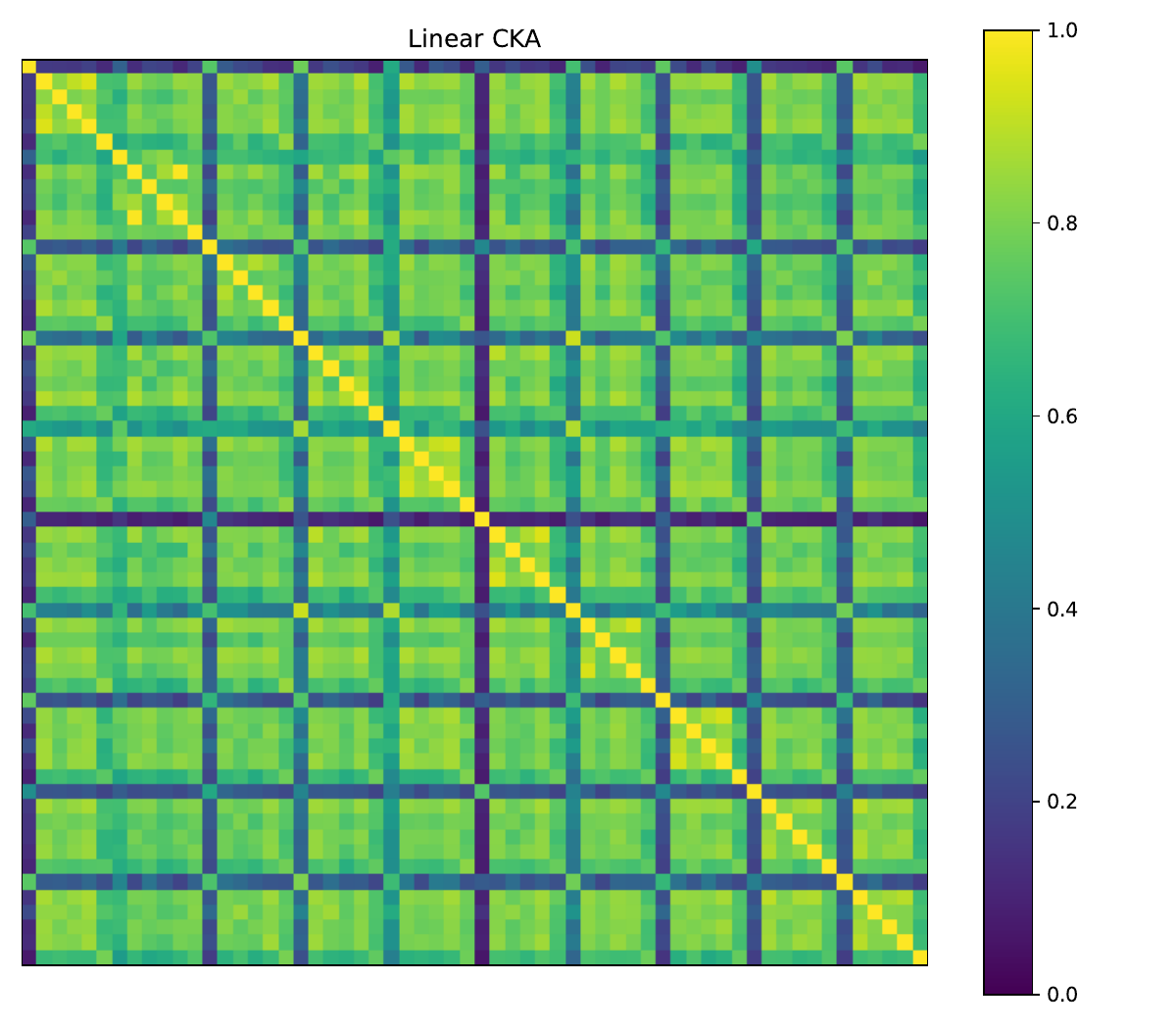}\\
{\small (b) CKA across matched models}
\end{minipage}
\caption{Waterbirds representation diagnostics after supervised-performance
matching. Models with comparable test accuracy occupy different regions of
nuisance decodability, transformation stability, effective rank, and CKA
similarity.}
\label{fig:waterbirds-diagnostics}
\end{figure}

Unlike the exact witnesses above, the Waterbirds study probes a matched-performance regime rather than a fixed predictor fiber. They instead show a matched-performance regime in which different constraints select different representation diagnostics at comparable supervised accuracy and loss.

The Waterbirds results also separate representation properties from predictor-level robustness. Worst-group accuracy does not move monotonically with nuisance probe accuracy. Augmentation-invariant training has high background decodability but better worst-group accuracy than ERM, whereas VIB reduces background decodability without improving worst-group accuracy. Thus nuisance information in the representation, final predictor behavior, and group robustness are related but distinct quantities.

\section{Discussion and Limitations}
\label{sec:discussion-limitations}

The fiber/descent criterion separates predictor-level evidence from representation-level conclusions. Supervised predictive behavior determines only those representation properties that are invariant across the relevant admissible predictor fiber. Properties such as compression, invariance, nuisance removal, equivariance, or semantic accessibility therefore require additional structure when they are not fiber-invariant. Architectural restrictions, bottlenecks, regularization, augmentation schemes, auxiliary objectives, causal assumptions, optimization bias, and direct representation-level measurements can serve this role by restricting the admissible class, selecting among predictor-equivalent factorizations, or measuring the representation itself. 

This dependence on the admissible class is essential. Optimizers, architectures, and initialization schemes can be viewed as selection mechanisms over the admissible factorization class. The present criterion specifies what must be invariant before such selection mechanisms are invoked. Restricted model classes may exclude the predictor-preserving augmentations used in the obstruction, and additional assumptions may make particular representation properties identifiable. The non-identifiability claim is therefore conditional: predictor-level evidence supports a representation-level property only when that property is invariant across the admissible alternatives.

The experiments instantiate this separation without exhausting the space of possible representation-learning mechanisms. The exact predictor-preserving witnesses show that common representation diagnostics can change while the predictor is fixed by construction. The Waterbirds experiments show that matched supervised performance can coexist with different representation diagnostics under different constraints. These diagnostics measure specific properties---linear decodability, effective rank, CKA similarity, nuisance decodability, and transformation stability---not a complete description of representation geometry or semantics. Broader studies across datasets, architectures, objectives, and modalities would further clarify how training procedures select representatives within predictor-equivalent or near-equivalent regimes.

\section{Conclusion}
\label{sec:conclusion}

This paper characterized when representation-level properties of supervised
factorized predictors are determined by predictor-level evidence. For the
projection
\[
\Pi(h,c)=[c\circ h]_{P_X},
\]
a representation property is identifiable from the induced predictor exactly
when it is constant on the fibers of $\Pi$, equivalently when it descends to a
well-defined property of the predictor. Since scalar supervised risk is coarser
than the induced predictor, any property that varies within a predictor fiber
is not identifiable from supervised risk alone.

Predictor-preserving augmentation gives a canonical witness of such variation:
unused admissible information can be appended to the representation while the
head ignores it. This leaves prediction unchanged but can alter compression,
invariance, nuisance information, equivariance, or semantic accessibility. The
empirical diagnostics instantiate this distinction through exact
predictor-preserving witnesses and matched-performance neural-network studies.

The criterion identifies where evidence for representation-level claims must enter: through architectural restrictions, regularization, bottlenecks, augmentation, auxiliary objectives, multiple environments, causal assumptions, optimization bias, or direct representation-level measurement. Supervised predictive behavior can support claims about
the induced predictor; claims about the internal representation require
architectural restrictions, regularization, bottlenecks, augmentation,
auxiliary objectives, multiple environments, causal assumptions, optimization
bias, or direct representation-level measurement.

\bibliographystyle{unsrtnat}
\bibliography{references}

\clearpage
\appendix

\section{Conceptual Illustrations of the Fiber Obstruction}
\label{app:conceptual-figures}

This appendix gives implementation details for the exact algebraic witnesses summarized in Section~\ref{subsec:experiments-exact-witnesses}. These witnesses are intentionally constructed: the auxiliary coordinate is appended to the representation, and the task head is defined to ignore it. These algebraic witnesses fix the supervised predictor exactly and measure how standard representation diagnostics vary within the resulting predictor fiber.

\subsection{Factorization, fibers, and descent}
\label{app:conceptual-fibers}

Figure~\ref{fig:app-supervised-risk-factorization} gives the most direct view of the supervised setup. A representation map $h$ and prediction head $c$ compose to form the predictor $f=c\circ h$, and the supervised risk evaluates the final prediction.

\begin{figure}[htbp]
\centering
\includegraphics[width=0.78\linewidth]{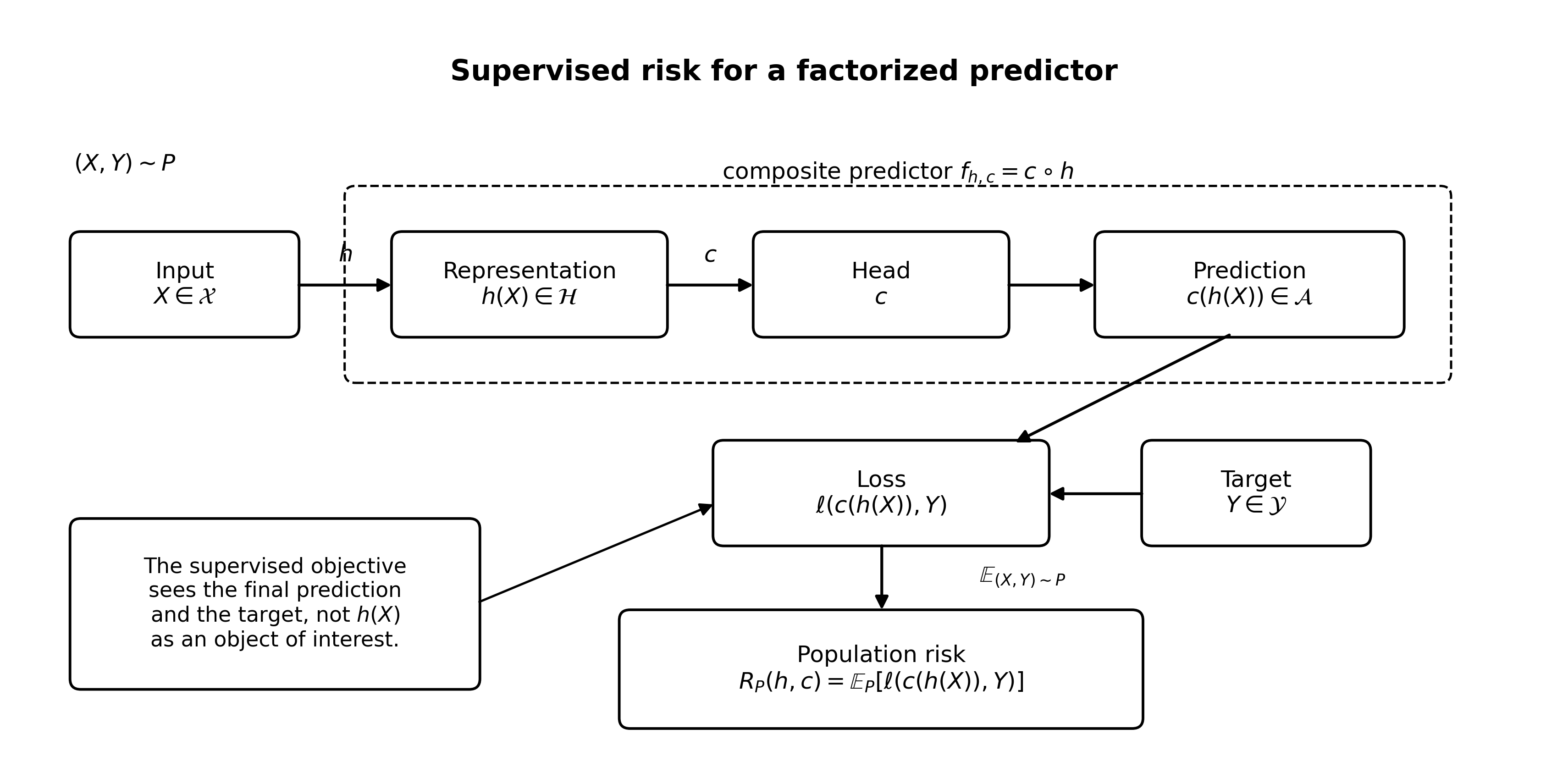}
\caption{Supervised risk for a factorized predictor. The supervised objective evaluates the composite prediction $f(x)=c(h(x))$. It does not directly evaluate the internal representation $h(x)$ or the particular factorization by which $f$ is implemented.}
\label{fig:app-supervised-risk-factorization}
\end{figure}

Figure~\ref{fig:app-fiber-descent-criterion} complements Theorem~\ref{thm:fiber-descent}. A representation property is identifiable from the induced predictor exactly when it is constant along each predictor fiber; this is the condition under which the property is already a property of the induced predictor.

\begin{figure}[htbp]
\centering
\includegraphics[width=0.90\linewidth]{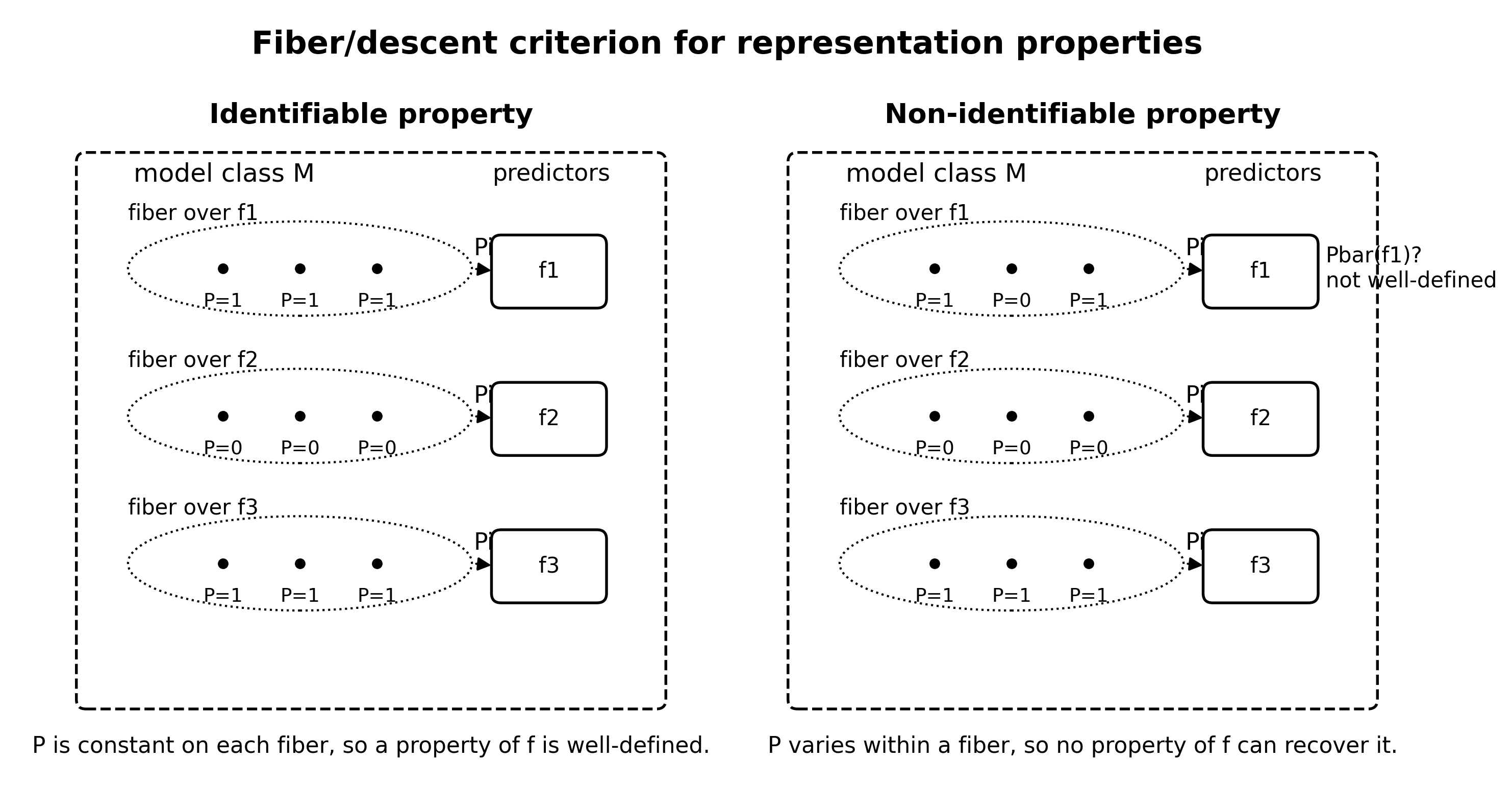}
\caption{The fiber/descent criterion. A representation property is identifiable from the induced predictor exactly when it is invariant across all admissible factorizations that induce the same predictor. If the property varies within a fiber of $\Pi:(h,c)\mapsto c\circ h$, then it cannot be recovered from predictor-level evidence.}
\label{fig:app-fiber-descent-criterion}
\end{figure}

\clearpage

\section{Additional Algebraic Empirical Case Studies}
\label{app:experiments}

This appendix gives implementation details for the exact predictor-preserving witnesses summarized in Section~\ref{subsec:experiments-exact-witnesses}.

Each case study begins with a trained supervised predictor
\begin{equation}
f(x)=c(h(x)),
\end{equation}
where $h$ is a penultimate representation and $c$ is the final prediction head. An augmented representation is then constructed as
\begin{equation}
h^q(x)=(h(x),q(x)),
\end{equation}
where $q(x)$ is an admissible auxiliary statistic or sample-level side-information field, and the augmented head is defined by
\begin{equation}
c^q(u,v)=c(u).
\end{equation}
Consequently, for every input $x$,
\begin{equation}
(c^q\circ h^q)(x)
=
c^q(h(x),q(x))
=
c(h(x))
=
(c\circ h)(x).
\end{equation}
Thus the original and augmented systems are exactly predictor-equivalent. Any change in a representation-level diagnostic after augmentation occurs inside a fixed predictor fiber. In the formal statement, $q$ is a measurable statistic of the selected input space. The case studies sometimes use annotations, metadata fields, or transformation identifiers as admissible side information; under a raw-pixel convention, this corresponds to an enlarged observation space or a model class that accepts such side information.

Three families of representation-level diagnostics are studied. First, on CelebA, semantic attributes are made linearly accessible from the representation while preserving the supervised predictor exactly. Second, on CIFAR-10 and STL-10, transformation sensitivity and transformation decodability are changed under fixed class prediction. Third, on PACS and OfficeHome, domain information is made perfectly decodable while preserving the class predictor exactly. The appendix records the case-study details, auxiliary controls, statistical procedures, secondary diagnostics, and empirical summaries supporting this representation-level interpretation.

\subsection{Case Study I: Semantic Attribute Accessibility on CelebA}
\label{app:case-study-celeba}

This case study gives a designed constructed semantic-accessibility witness. A non-target CelebA attribute is appended explicitly to the frozen representation, while the original \texttt{Smiling} head is extended to ignore the appended coordinate. It is therefore expected that a probe can recover the appended attribute. The point is that this change in semantic accessibility occurs between two systems with exactly the same \texttt{Smiling} predictor. Thus supervised accuracy on \texttt{Smiling} cannot by itself certify that non-target semantic attributes are absent or inaccessible from the representation.

\subsubsection{Setup}
\label{app:celeba-setup}

Let $X$ denote a CelebA face image and let $Y$ be the supervised target attribute. The supervised task is binary prediction of
\begin{equation}
Y=\texttt{Smiling}.
\end{equation}
All remaining CelebA attributes are treated as auxiliary semantic attributes and admissible side information for the diagnostic. For each such attribute, denoted $q_j(X)\in\{0,1\}$, a predictor-preserving augmentation of the learned representation is constructed.

A ResNet--18 classifier is trained with ImageNet initialization. Let
\begin{equation}
f(x)=c(h(x))
\end{equation}
be the resulting supervised predictor, where $h(x)\in\mathbb R^{512}$ is the penultimate-layer representation and $c$ is the final linear classification head. The model is trained on the supervised target $Y=\texttt{Smiling}$. The auxiliary attributes are excluded from supervised classifier training.

After training, the representation $h$ and head $c$ are frozen. For each auxiliary attribute $q_j$, define
\begin{equation}
h^{q_j}(x)=(h(x),q_j(x))\in\mathbb R^{513},
\end{equation}
and
\begin{equation}
c^{q_j}(u,v)=c(u).
\end{equation}
Therefore, for every input $x$,
\begin{equation}
c^{q_j}(h^{q_j}(x))=c(h(x)).
\end{equation}

Two kinds of quantities are evaluated. First, predictor preservation is verified by reporting supervised test loss, accuracy, balanced accuracy, AUC, prediction disagreement, and maximum logit difference. Second, attribute accessibility is measured by training linear probes to predict each auxiliary attribute $q_j$ from either $h(X)$ or $h^{q_j}(X)$. Probe performance is measured by AUC and balanced accuracy.

Two controls are included. The first appends an independent random scalar coordinate. The second appends a shuffled version of the attribute. These controls test whether the change in accessibility is caused by meaningful attribute information instead of dimensionality alone.

For uncertainty quantification, paired bootstrap confidence intervals with 2000 bootstrap resamples and paired permutation tests with 2000 permutations are used. Since the same test examples are evaluated before and after augmentation, all tests are paired. The tests are corrected across the 39 auxiliary attributes using the Benjamini--Hochberg procedure at false discovery rate $0.05$.

\subsubsection{Predictor preservation}
\label{app:celeba-predictor-preservation}

Table~\ref{tab:app-celeba-predictor-preservation} verifies exact predictor preservation. The supervised classifier achieves test accuracy $0.9299$ and test AUC $0.9828$ on \texttt{Smiling}. After augmentation, the supervised test loss, accuracy, balanced accuracy, and AUC are identical before and after augmentation. Prediction disagreement and maximum logit difference are both zero.

\begin{table}[h]
\centering
\caption{Predictor preservation on CelebA for the supervised target $Y=\texttt{Smiling}$. The augmented representation is $h^{q_j}(x)=(h(x),q_j(x))$, and the augmented head is $c^{q_j}(u,v)=c(u)$.}
\label{tab:app-celeba-predictor-preservation}
\begin{tabular}{lcc}
\toprule
Quantity & Original $c\circ h$ & Augmented $c^{q_j}\circ h^{q_j}$ \\
\midrule
Test loss & 0.1716 & 0.1716 \\
Test accuracy & 0.9299 & 0.9299 \\
Test balanced accuracy & 0.9299 & 0.9299 \\
Test AUC & 0.9828 & 0.9828 \\
Prediction disagreement & -- & 0.0000 \\
Maximum logit difference & -- & 0.0000 \\
\bottomrule
\end{tabular}
\end{table}

This table is the empirical counterpart of the predictor-fiber construction. For every auxiliary attribute $q_j$, the original and augmented systems lie in the same predictor fiber. Therefore, any representation-level quantity that differs between $h$ and $h^{q_j}$ is not determined by the supervised predictor.

\subsubsection{Aggregate attribute-accessibility results}
\label{app:celeba-aggregate}

For each of the 39 non-target CelebA attributes, a linear probe is trained on the frozen original representation $h(X)$ and another linear probe is trained on the augmented representation $h^{q_j}(X)=(h(X),q_j(X))$. Since $q_j(X)$ is explicitly included as a coordinate of $h^{q_j}(X)$, perfect recovery from $h^{q_j}(X)$ is expected. The important point is that this change in representation-level accessibility occurs while the supervised predictor is exactly unchanged.

Table~\ref{tab:app-celeba-attribute-summary} reports the aggregate results. The augmented representation achieves AUC $1.000$ for every appended attribute. The mean increase in probe AUC is $0.137$, with median increase $0.120$, minimum increase $0.023$, and maximum increase $0.320$. Random-coordinate and shuffled-coordinate controls produce essentially zero mean change in AUC, showing that the effect is not explained by increasing the representation dimension.

\begin{table}[h]
\centering
\caption{Aggregate CelebA attribute-accessibility results over the 39 attributes excluded from the supervised target. $\Delta$AUC denotes probe AUC after augmentation minus probe AUC before augmentation.}
\label{tab:app-celeba-attribute-summary}
\begin{tabular}{lc}
\toprule
Quantity & Value \\
\midrule
Number of auxiliary attributes & 39 \\
Test examples & 19,962 \\
Mean $\Delta$AUC, true attribute augmentation & 0.137 \\
Median $\Delta$AUC, true attribute augmentation & 0.120 \\
Minimum $\Delta$AUC, true attribute augmentation & 0.023 \\
Maximum $\Delta$AUC, true attribute augmentation & 0.320 \\
Mean $\Delta$ balanced accuracy & 0.353 \\
Augmented probe AUC & 1.000 for all attributes \\
Mean $\Delta$AUC, random-coordinate control & $5.7\times 10^{-5}$ \\
Mean $\Delta$AUC, shuffled-coordinate control & $-6.9\times 10^{-5}$ \\
FDR-significant attributes at level 0.05 & 39/39 \\
\bottomrule
\end{tabular}
\end{table}

The controls are important. Adding an arbitrary scalar coordinate leaves attribute accessibility unchanged. What changes the representation property is the semantic information carried by the appended coordinate. Since the supervised predictor is exactly preserved, this semantic information is invisible to supervised risk.

\subsubsection{Representative attribute-level results}
\label{app:celeba-selected}

Table~\ref{tab:app-celeba-selected-attributes} reports representative attribute-level results. The first ten rows are the largest AUC increases, and the last five rows are the smallest AUC increases. The largest increases occur for attributes that are moderately accessible from the original \texttt{Smiling} representation, such as \texttt{Big\_Lips}, \texttt{Pointy\_Nose}, and \texttt{Oval\_Face}. Attributes that are already highly accessible from the original representation, such as \texttt{Bald}, \texttt{Wearing\_Hat}, \texttt{Eyeglasses}, and \texttt{Male}, also show statistically significant increases after predictor-preserving augmentation.

\begin{table*}[h]
\centering
\caption{Representative CelebA attribute-accessibility results. The first ten rows are the largest AUC increases; the last five rows are the smallest AUC increases. All augmentations preserve the supervised \texttt{Smiling} predictor exactly. Confidence intervals are paired bootstrap 95\% intervals for $\Delta$AUC.}
\label{tab:app-celeba-selected-attributes}
\begin{tabular}{lrrrrrrr}
\toprule
Attribute $q_j$ & Prev. & AUC$(h)$ & AUC$(h^{q_j})$ & $\Delta$AUC & 95\% CI & Rand. AUC & Shuf. AUC \\
\midrule
\texttt{Big\_Lips} & 0.327 & 0.680 & 1.000 & 0.320 & [0.312, 0.328] & 0.681 & 0.680 \\
\texttt{Pointy\_Nose} & 0.286 & 0.714 & 1.000 & 0.286 & [0.279, 0.294] & 0.713 & 0.714 \\
\texttt{Oval\_Face} & 0.296 & 0.731 & 1.000 & 0.269 & [0.262, 0.276] & 0.731 & 0.731 \\
\texttt{Narrow\_Eyes} & 0.149 & 0.749 & 1.000 & 0.251 & [0.242, 0.260] & 0.750 & 0.750 \\
\texttt{Straight\_Hair} & 0.210 & 0.763 & 1.000 & 0.237 & [0.230, 0.245] & 0.764 & 0.763 \\
\texttt{Brown\_Hair} & 0.180 & 0.772 & 1.000 & 0.228 & [0.220, 0.237] & 0.771 & 0.770 \\
\texttt{Wearing\_Necklace} & 0.138 & 0.774 & 1.000 & 0.226 & [0.218, 0.234] & 0.773 & 0.773 \\
\texttt{Bushy\_Eyebrows} & 0.130 & 0.785 & 1.000 & 0.215 & [0.206, 0.225] & 0.785 & 0.784 \\
\texttt{Wearing\_Earrings} & 0.207 & 0.799 & 1.000 & 0.201 & [0.194, 0.209] & 0.800 & 0.800 \\
\texttt{Arched\_Eyebrows} & 0.284 & 0.803 & 1.000 & 0.197 & [0.191, 0.203] & 0.803 & 0.803 \\
\midrule
\texttt{Bald} & 0.021 & 0.977 & 1.000 & 0.023 & [0.019, 0.029] & 0.977 & 0.977 \\
\texttt{Wearing\_Hat} & 0.042 & 0.975 & 1.000 & 0.025 & [0.020, 0.031] & 0.975 & 0.974 \\
\texttt{Eyeglasses} & 0.065 & 0.973 & 1.000 & 0.027 & [0.023, 0.032] & 0.973 & 0.972 \\
\texttt{Male} & 0.386 & 0.949 & 1.000 & 0.051 & [0.049, 0.054] & 0.949 & 0.950 \\
\texttt{Wearing\_Lipstick} & 0.522 & 0.938 & 1.000 & 0.062 & [0.059, 0.066] & 0.937 & 0.938 \\
\bottomrule
\end{tabular}
\end{table*}

The statistical tests quantify the stability of the measured probe changes across the finite test set. They are not needed to establish predictor preservation, which follows algebraically from the construction. The reported adjusted $p$-values are at the resolution limit of the 2000-permutation test, approximately $5\times 10^{-4}$.

\subsubsection{Figures and interpretation}
\label{app:celeba-figures}

Figure~\ref{fig:app-celeba-delta-auc} shows the AUC increase for each appended attribute. Figure~\ref{fig:app-celeba-orig-vs-aug} compares probe AUC before and after augmentation; all points move to AUC $1.0$ after the corresponding attribute is appended. Figure~\ref{fig:app-celeba-controls} visualizes the contrast between true attribute augmentation and the controls. The true augmentations yield positive changes for all attributes, whereas the random and shuffled controls remain concentrated near zero.

\begin{figure}[h]
\centering
\includegraphics[width=0.95\linewidth]{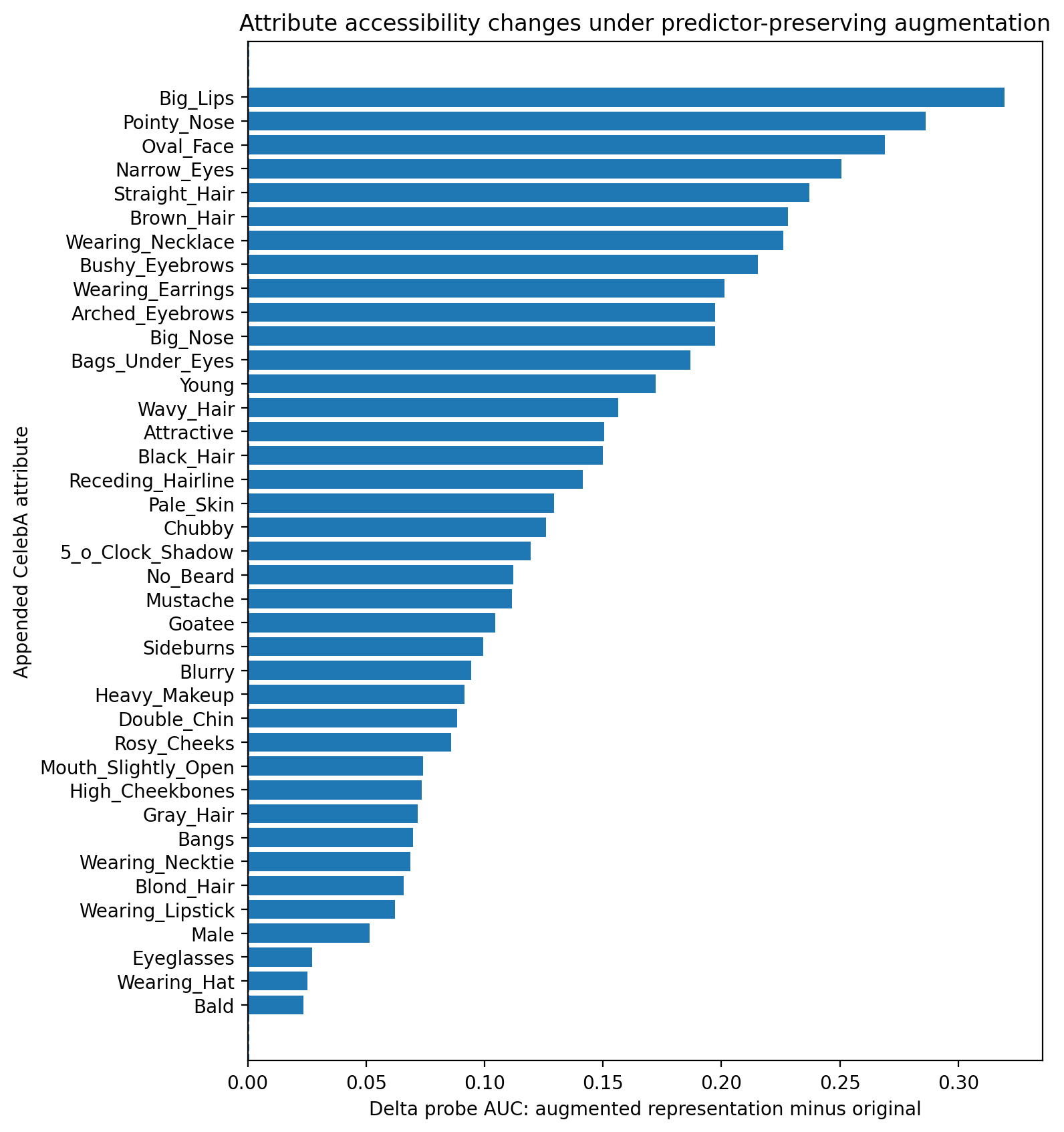}
\caption{Increase in auxiliary-attribute probe AUC after predictor-preserving augmentation on CelebA. The supervised target is \texttt{Smiling}; each bar corresponds to appending one non-target CelebA attribute $q_j$ to the representation. The supervised predictor is unchanged for all bars.}
\label{fig:app-celeba-delta-auc}
\end{figure}

\begin{figure}[h]
\centering
\includegraphics[width=0.75\linewidth]{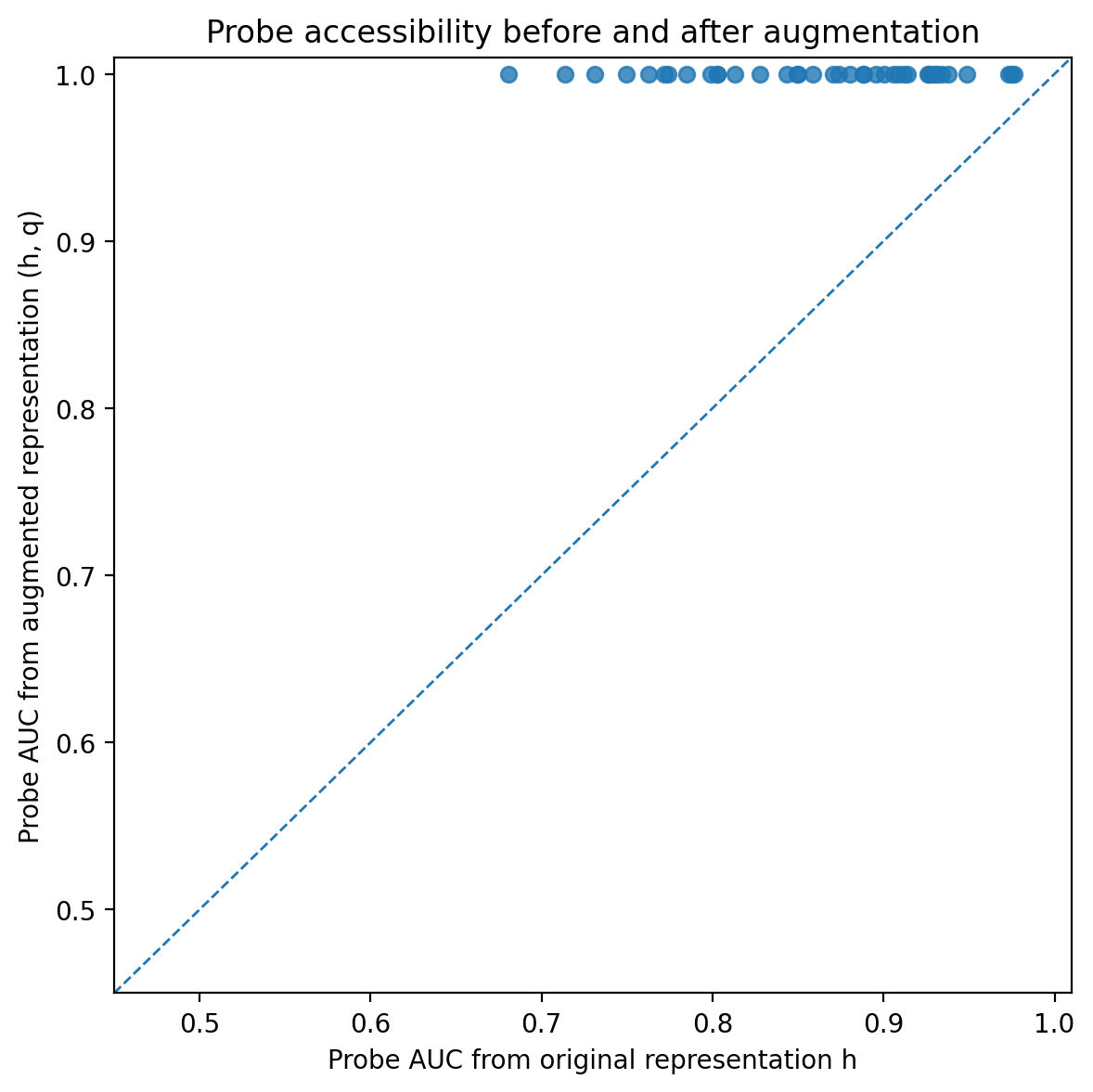}
\caption{Probe AUC from the original representation $h$ versus the augmented representation $h^{q_j}=(h,q_j)$. Augmentation makes the appended attribute directly accessible while leaving the supervised \texttt{Smiling} predictor exactly fixed.}
\label{fig:app-celeba-orig-vs-aug}
\end{figure}

\begin{figure}[h]
\centering
\includegraphics[width=0.8\linewidth]{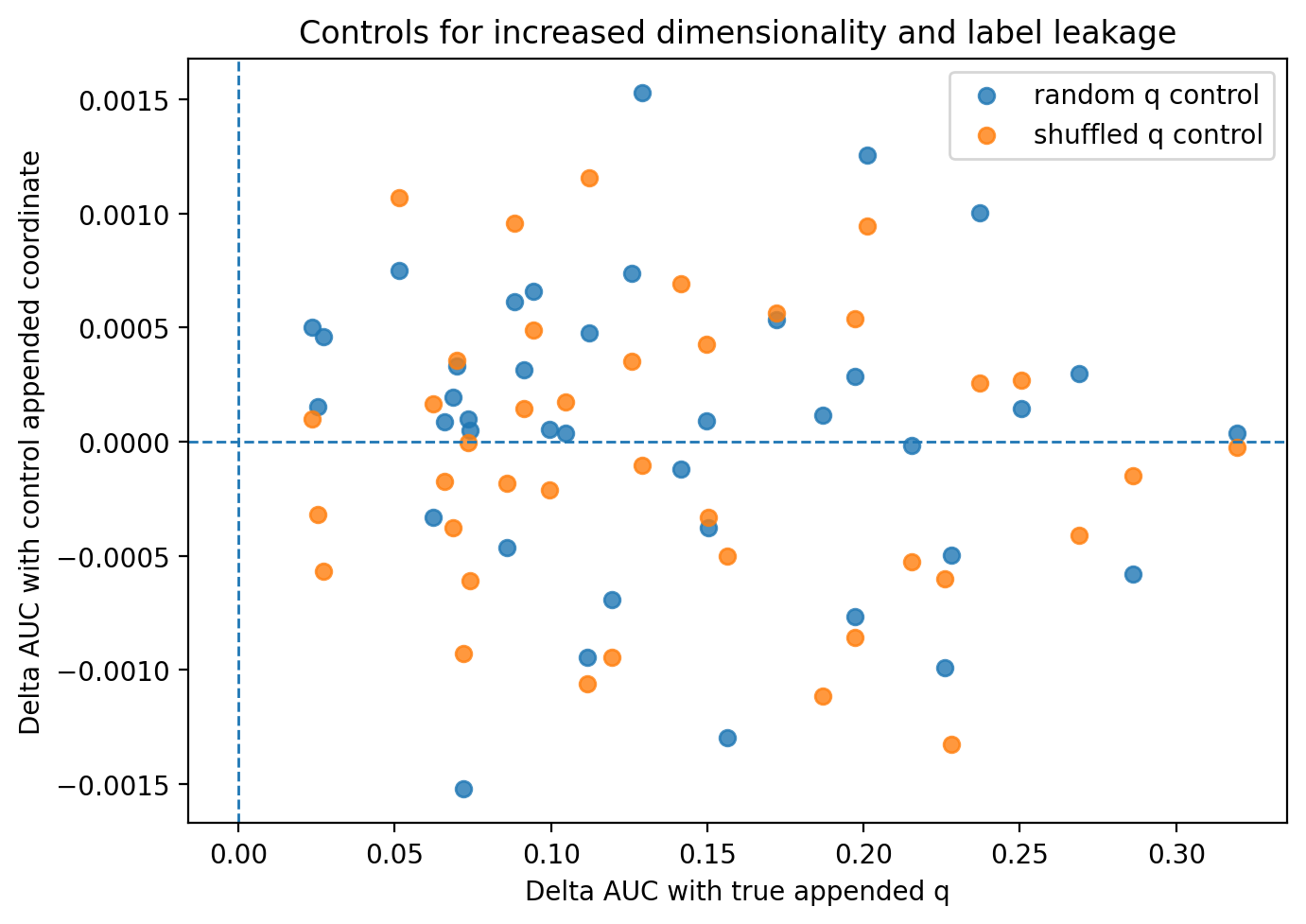}
\caption{CelebA controls. Appending an independent random coordinate or a shuffled attribute coordinate produces negligible changes in probe AUC, whereas appending the true attribute makes the corresponding semantic attribute directly accessible. The supervised \texttt{Smiling} predictor is unchanged in all cases.}
\label{fig:app-celeba-controls}
\end{figure}

Thus semantic-attribute accessibility changes under an exactly fixed \texttt{Smiling} predictor.

\subsection{Case Study II: Transformation Sensitivity Under Fixed Prediction}
\label{app:case-study-transform}

The second diagnostic gives a finite-sample witness for transformation sensitivity
and transformation decodability. The supervised task is the original
class-prediction problem on CIFAR-10 and STL-10. Given a trained classifier $f=c\circ h$, an augmented representation is formed
as $h^q(x)=(h(x),q(x))$, where $q$ carries transformation-sensitive
information. The augmented head is defined by $c^q(u,v)=c(u)$, so the added
coordinate is ignored by the class predictor. Thus class
predictions are fixed pointwise by construction. The construction fixes the class predictor pointwise and varies only unused representation coordinates. The resulting diagnostics quantify representation-level movement within an empirical predictor fiber, including changes in transformation decodability, invariance distance, and effective rank.

\subsubsection{Datasets and models}
\label{app:transform-datasets-models}

Transformation sensitivity is evaluated on CIFAR-10 and STL-10. Both are image-classification datasets with ten object categories, but they differ in image resolution, sample size, and visual variability. CIFAR-10 provides a relatively stable classification benchmark, whereas STL-10 is more challenging and contains higher-resolution images.

For each dataset, a supervised convolutional classifier is trained and the penultimate representation is extracted:
\begin{equation}
h(x)\in\mathbb R^{512}.
\end{equation}
The final linear classifier is used as the prediction head $c$. The supervised models are trained for the original class-prediction task. The auxiliary quantities $q(x)$ used below are excluded from the supervised head.

The experiment considers four input transformations:
\begin{enumerate}
    \item horizontal flip,
    \item color jitter,
    \item crop/resize,
    \item rotation.
\end{enumerate}
These transformations are used for representation diagnostics and leave the definition of the supervised classifier unchanged.

\subsubsection{Auxiliary representation augmentations}
\label{app:transform-augmentations}

Three meaningful auxiliary statistics and two controls are considered. Each auxiliary statistic is appended to the original representation while the prediction head is forced to ignore it.

The meaningful augmentations are:
\begin{enumerate}
    \item \textbf{Color histogram.} A low-level color histogram is appended to the representation. This statistic is expected to be sensitive to color perturbations.
    \item \textbf{Edge-orientation histogram.} An edge-orientation histogram is appended to the representation. This statistic is expected to be sensitive to geometric transformations such as flips, crops, and rotations.
    \item \textbf{Transformation embedding.} A learned auxiliary embedding is appended. This embedding is trained to encode transformation-related information and is therefore expected to be strongly transformation-sensitive.
\end{enumerate}

The controls are:
\begin{enumerate}
    \item \textbf{Random control.} Random auxiliary coordinates are appended. This controls for the effect of increased dimensionality.
    \item \textbf{Shuffled control.} A shuffled auxiliary feature is appended. This control is reported in the raw outputs as a secondary diagnostic because in this construction it still exhibits substantial transformation-predictive structure.
\end{enumerate}

For every variant,
\begin{equation}
h^q(x)=(h(x),q(x)),
\qquad
c^q(u,v)=c(u).
\end{equation}
Thus all augmented systems preserve the original class predictor exactly. The augmented coordinates are present in the representation but are unused by the supervised predictor.

\subsubsection{Evaluation metrics}
\label{app:transform-metrics}

Exact preservation of the supervised predictor is first verified. The reported quantities are task accuracy, cross-entropy loss, prediction-disagreement rate, and maximum absolute logit difference. The prediction-disagreement rate is
\begin{equation}
D_{\mathrm{pred}}
=
\frac{1}{n}\sum_{i=1}^n
\mathbf 1\{c(h(x_i))\neq c^q(h^q(x_i))\}.
\end{equation}
For exact predictor preservation,
\begin{equation}
D_{\mathrm{pred}}=0.
\end{equation}

Representation-level changes are then measured. For a transformation $g$, the empirical invariance distance is
\begin{equation}
I_g(h)
=
\frac{1}{n}
\sum_{i=1}^n
\frac{\|h(g\cdot x_i)-h(x_i)\|_2}{\|h(x_i)\|_2+\epsilon}.
\end{equation}
Larger values indicate greater transformation sensitivity, or weaker invariance. The reported change is
\begin{equation}
\Delta I_g=I_g(h^q)-I_g(h).
\end{equation}

Linear probes are also trained to predict the applied transformation from the representation. Let $A_{\mathrm{probe}}(h)$ denote transformation-probe accuracy from the original representation and $A_{\mathrm{probe}}(h^q)$ the corresponding accuracy from the augmented representation.

Finally, representation dimension and effective rank are reported. Effective rank is used as a coarse compression proxy: it measures the spread of representation variance across directions and detects representation-level changes under fixed prediction.

For paired invariance changes, nonparametric bootstrap confidence intervals and sign-flip permutation tests are computed. Empirical power is also estimated by subsampling at sample sizes
\begin{equation}
64,\;128,\;256,\;512,\;1024,\;2048.
\end{equation}
The statistical tests quantify the stability of the measured representation-level effects. The exact equality of the supervised predictor follows algebraically from the construction.

\subsubsection{Predictor preservation}
\label{app:transform-predictor-preservation}

Table~\ref{tab:app-transform-predictor-preservation} verifies that the supervised predictor is preserved exactly. Across all augmentation variants, task accuracy, cross-entropy loss, predicted labels, and logits are unchanged.

\begin{table}[h]
\centering
\caption{Predictor preservation for transformation-sensitivity diagnostics. The augmented representations are of the form $h^q(x)=(h(x),q(x))$, and the augmented head is $c^q(u,v)=c(u)$.}
\label{tab:app-transform-predictor-preservation}
\begin{tabular}{lccccc}
\toprule
Dataset & Accuracy $h$ & Accuracy $h^q$ & CE $h$ & CE $h^q$ & Disagreement \\
\midrule
CIFAR-10 & $0.9313$ & $0.9313$ & $0.4124$ & $0.4124$ & $0.0000$ \\
STL-10   & $0.7636$ & $0.7636$ & $0.7307$ & $0.7307$ & $0.0000$ \\
\bottomrule
\end{tabular}
\end{table}

This table is central to the interpretation. The augmented representations differ from the original representation, but the supervised predictor does not. Therefore, any subsequent change in invariance, transformation decodability, effective rank, or dimensionality is a representation-level change under a fixed supervised input-output map.

\subsubsection{Transformation information becomes more accessible}
\label{app:transform-probes}

Table~\ref{tab:app-transform-probe} reports transformation-probe accuracy for predicting the applied transformation from the original and augmented representations. The learned transformation embedding produces a large increase in probe accuracy on both datasets, while the random control does not improve probe accuracy.

\begin{table}[h]
\centering
\caption{Transformation-probe accuracy under fixed class prediction. The learned transformation embedding increases linear accessibility of transformation information, while the random control does not.}
\label{tab:app-transform-probe}
\begin{tabular}{lccccc}
\toprule
Dataset & $h$ & $h+$color & $h+$edge & $h+$transform emb. & $h+$random \\
\midrule
CIFAR-10 & $0.5948$ & $0.6245$ & $0.6033$ & $0.7907$ & $0.5940$ \\
STL-10   & $0.5732$ & $0.5935$ & $0.5802$ & $0.7034$ & $0.5685$ \\
\bottomrule
\end{tabular}
\end{table}

The original representations already contain nontrivial transformation information. This is expected: supervised image classifiers need not discard transformation-sensitive information. However, appending transformation-sensitive auxiliary information substantially increases linear accessibility. On CIFAR-10, the learned transformation embedding increases probe accuracy from $59.48\%$ to $79.07\%$. On STL-10, it increases probe accuracy from $57.32\%$ to $70.34\%$. These increases occur while task predictions and losses are exactly unchanged.

The random control has the same augmented dimensionality as the learned transformation embedding and leaves transformation-probe accuracy unchanged. The increase is caused by appending transformation-informative coordinates and not by increasing representation dimension.

\begin{figure}[h]
\centering
\begin{minipage}{0.48\linewidth}
    \centering
    \includegraphics[width=\linewidth]{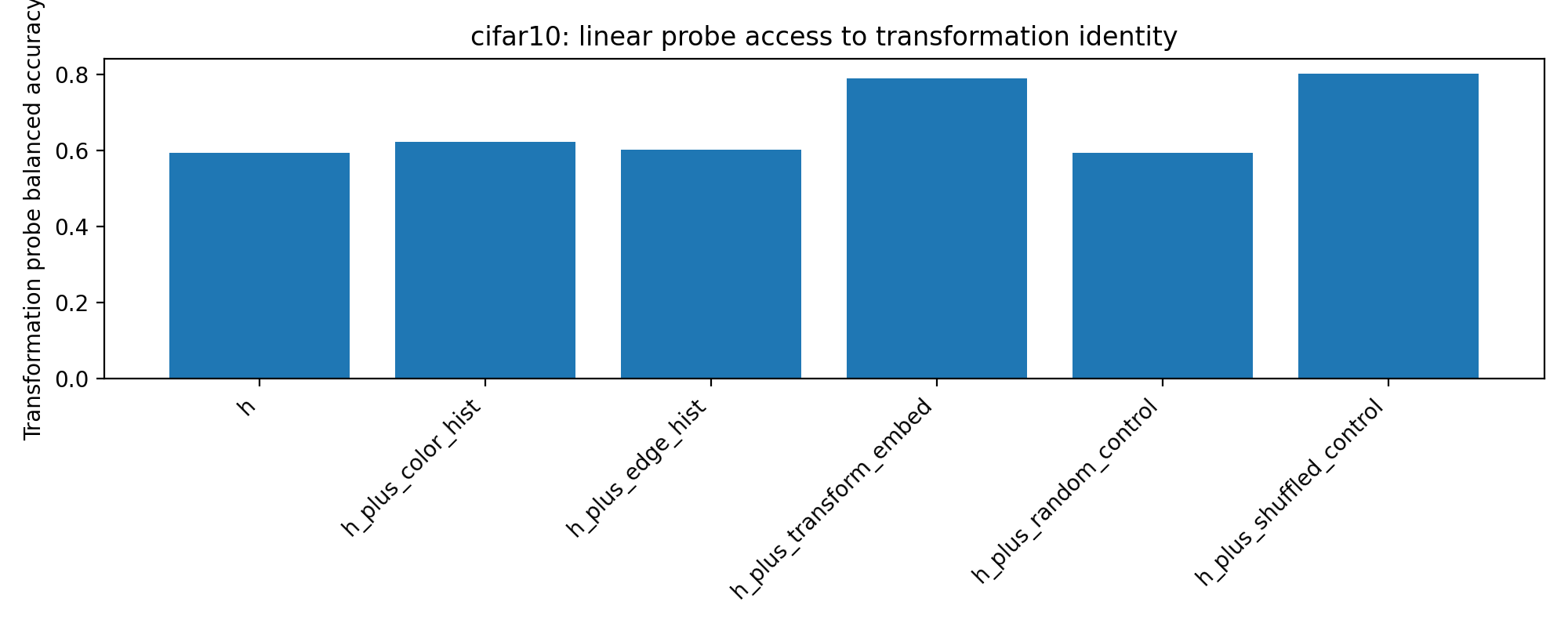}
    \caption*{CIFAR-10}
\end{minipage}
\hfill
\begin{minipage}{0.48\linewidth}
    \centering
    \includegraphics[width=\linewidth]{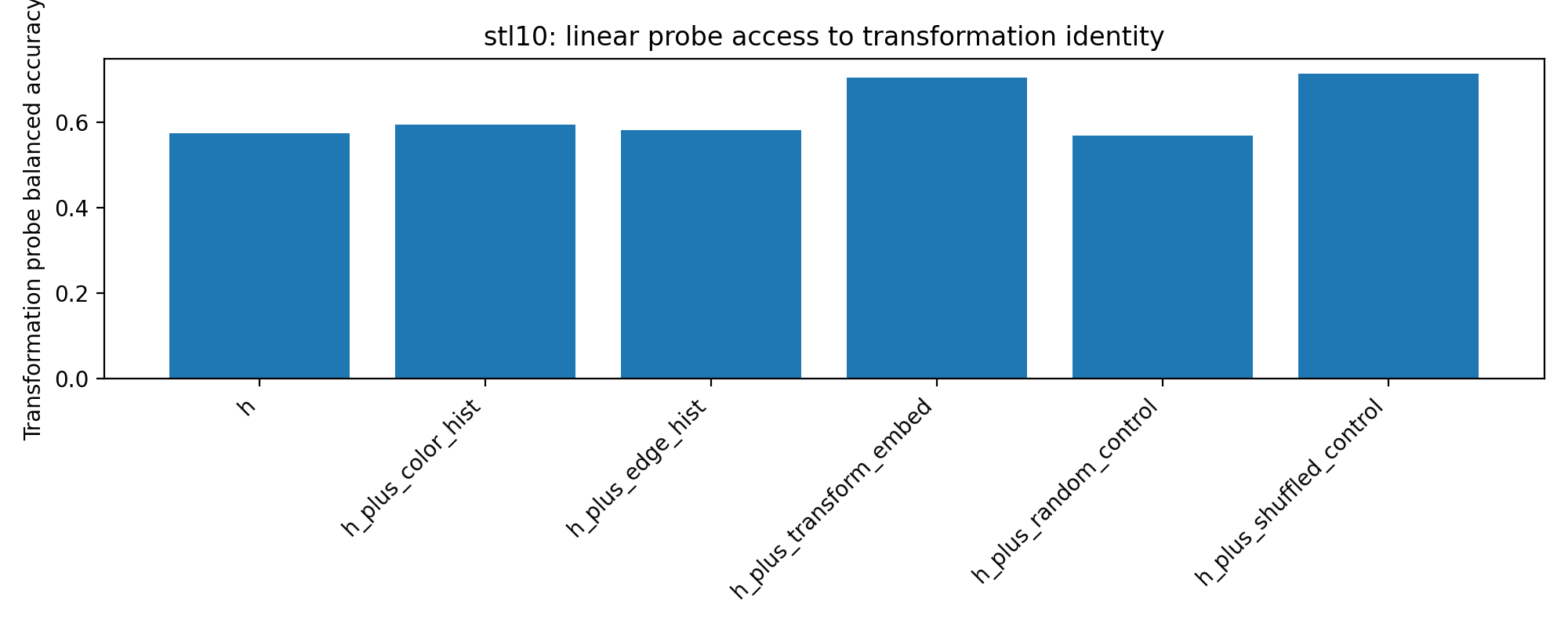}
    \caption*{STL-10}
\end{minipage}
\caption{Transformation-probe accuracy for the original and augmented representations. The supervised predictor is identical across all variants within each dataset. Nevertheless, transformation information becomes substantially more accessible after appending a learned transformation-sensitive embedding.}
\label{fig:app-transform-probe}
\end{figure}

\subsubsection{Transformation invariance changes under fixed prediction}
\label{app:transform-invariance}

Table~\ref{tab:app-transform-invariance-transform-embed} reports $\Delta I_g$ for the learned transformation embedding. Positive values indicate that the augmented representation is less invariant to the corresponding transformation.

\begin{table}[h]
\centering
\caption{Change in invariance distance after appending the learned transformation embedding. Positive values indicate increased transformation sensitivity. The supervised predictor is unchanged in every row.}
\label{tab:app-transform-invariance-transform-embed}
\begin{tabular}{lcccc}
\toprule
Dataset & Horizontal flip & Color jitter & Crop/resize & Rotation \\
\midrule
CIFAR-10 & $+0.0771$ & $+0.1538$ & $+0.1384$ & $+0.0544$ \\
STL-10   & $+0.1118$ & $+0.2300$ & $+0.2314$ & $+0.1203$ \\
\bottomrule
\end{tabular}
\end{table}

The learned transformation embedding consistently increases transformation sensitivity across both datasets and all transformations. The effect is especially large for color jitter and crop/resize. On STL-10, for example, the change is $+0.2300$ for color jitter and $+0.2314$ for crop/resize. These are substantial representation-level changes despite exact preservation of the supervised predictor.

The handcrafted augmentations provide interpretable controls. Color histograms strongly affect sensitivity to color jitter, whereas edge histograms more directly affect geometric transformations. Table~\ref{tab:app-transform-invariance-handcrafted} reports representative changes.

\begin{table}[h]
\centering
\caption{Representative changes in invariance distance for handcrafted auxiliary statistics. Color histograms primarily affect color perturbations, whereas edge histograms affect geometric transformations. The supervised predictor is unchanged in all rows.}
\label{tab:app-transform-invariance-handcrafted}
\begin{tabular}{llcc}
\toprule
Dataset & Transformation & $h+$color hist. & $h+$edge hist. \\
\midrule
CIFAR-10 & Horizontal flip & $-0.0059$ & $+0.0461$ \\
CIFAR-10 & Color jitter    & $+0.1349$ & $+0.0047$ \\
CIFAR-10 & Crop/resize     & $-0.0044$ & $+0.0010$ \\
CIFAR-10 & Rotation        & $-0.0374$ & $+0.0238$ \\
\midrule
STL-10 & Horizontal flip & $-0.0046$ & $+0.0629$ \\
STL-10 & Color jitter    & $+0.1542$ & $+0.0016$ \\
STL-10 & Crop/resize     & $-0.0064$ & $+0.0043$ \\
STL-10 & Rotation        & $-0.0253$ & $+0.0420$ \\
\bottomrule
\end{tabular}
\end{table}

These results show that different auxiliary coordinates change different representation-level invariance diagnostics. This is exactly the behavior predicted by the augmentation obstruction: the head can ignore the added coordinates, but representation-level properties remain sensitive to them.

\begin{figure}[h]
\centering
\begin{minipage}{0.48\linewidth}
    \centering
    \includegraphics[width=\linewidth]{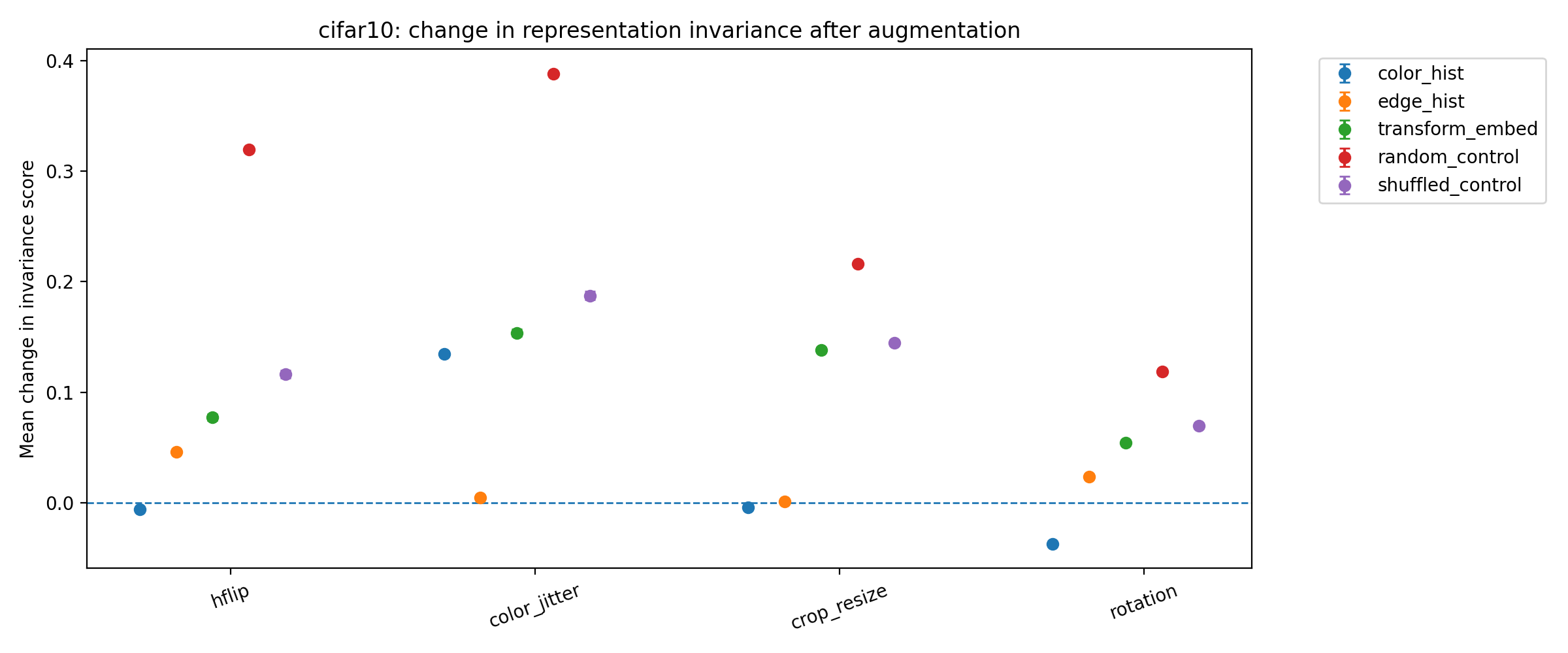}
    \caption*{CIFAR-10}
\end{minipage}
\hfill
\begin{minipage}{0.48\linewidth}
    \centering
    \includegraphics[width=\linewidth]{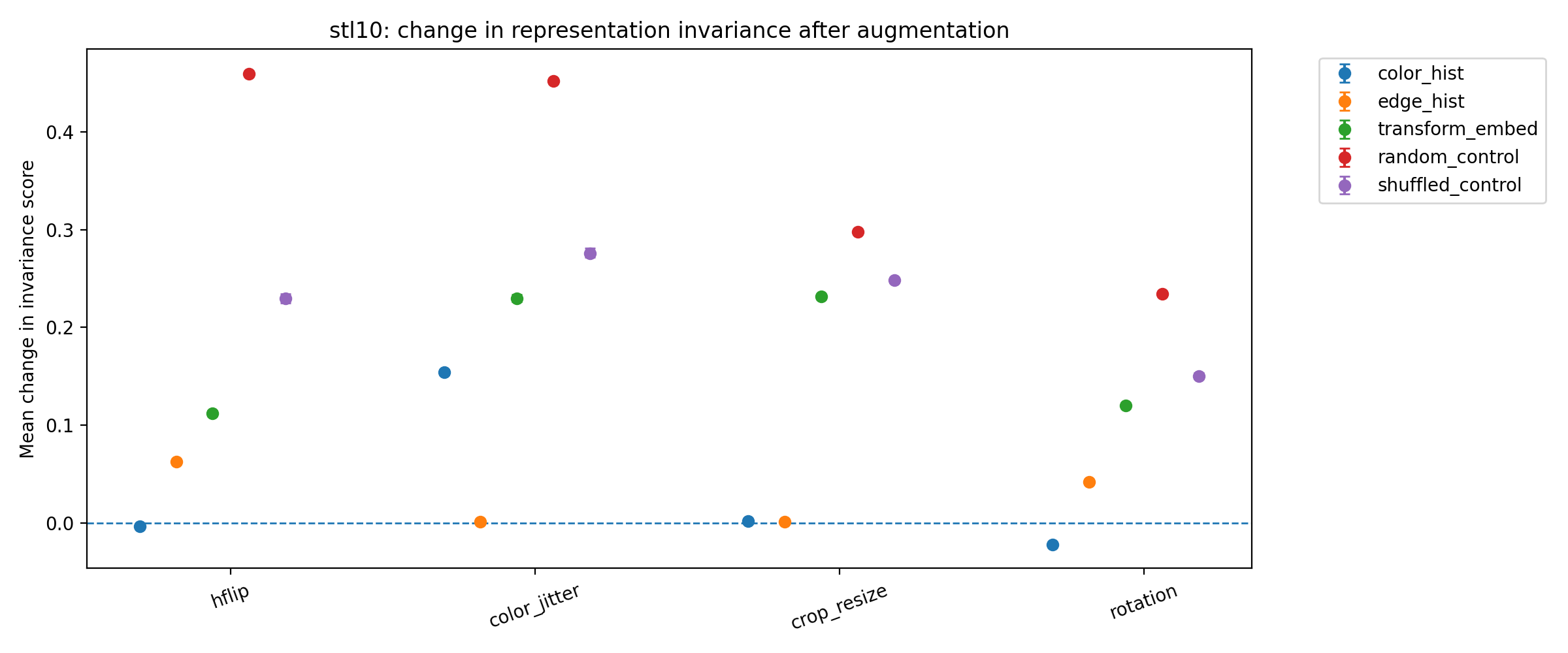}
    \caption*{STL-10}
\end{minipage}
\caption{Change in empirical invariance distance after representation augmentation. Positive values indicate increased transformation sensitivity. The augmented systems preserve the supervised predictor exactly, yet their representation-level transformation sensitivity changes.}
\label{fig:app-transform-invariance}
\end{figure}

\subsubsection{Compression proxies also change}
\label{app:transform-compression}

The same predictor-preserving intervention changes compression-related diagnostics. Table~\ref{tab:app-transform-rank} reports representation dimension and effective rank. The original representation has dimension $512$. Appending handcrafted or learned auxiliary statistics increases representation dimension and changes effective rank, while classifier outputs remain fixed.

\begin{table}[h]
\centering
\caption{Representation dimension and effective rank. Augmentation changes compression-related representation diagnostics while preserving the supervised predictor exactly.}
\label{tab:app-transform-rank}
\begin{tabular}{lccc}
\toprule
Dataset & Representation & Dimension & Effective rank \\
\midrule
CIFAR-10 & $h$ & $512$ & $15.90$ \\
CIFAR-10 & $h+$color hist. & $560$ & $21.39$ \\
CIFAR-10 & $h+$edge hist. & $528$ & $17.58$ \\
CIFAR-10 & $h+$transform emb. & $640$ & $20.11$ \\
CIFAR-10 & $h+$random & $640$ & $38.73$ \\
\midrule
STL-10 & $h$ & $512$ & $11.18$ \\
STL-10 & $h+$color hist. & $560$ & $18.30$ \\
STL-10 & $h+$edge hist. & $528$ & $12.48$ \\
STL-10 & $h+$transform emb. & $640$ & $15.93$ \\
STL-10 & $h+$random & $640$ & $46.01$ \\
\bottomrule
\end{tabular}
\end{table}

These results use effective rank as a diagnostic compression proxy. They illustrate a general point: compression-like representation diagnostics can change under a predictor-preserving transformation. Therefore, such diagnostics are not determined by the supervised predictor alone.

\subsubsection{Statistical stability}
\label{app:transform-statistical-stability}

The invariance changes are statistically stable. For all reported invariance comparisons, the bootstrap confidence intervals exclude zero. Sign-flip permutation tests yield $p=0.0005$ for the reported rows, using 2000 permutations. Subsampling-based empirical power estimates are close to one for most nontrivial effects, including the learned transformation embedding at small sample sizes.

For the learned transformation embedding, paired effect sizes are moderate to large. On CIFAR-10, paired Cohen's $d$ ranges from $0.51$ for horizontal flips to $1.18$ for crop/resize. On STL-10, it ranges from $0.92$ for horizontal flips to $2.30$ for crop/resize. Thus, the representation-level changes are not marginal numerical artifacts; they are stable effects under the diagnostic metrics.

The statistical analysis concerns the measured representation diagnostics, not the predictor-preservation claim. Predictor preservation follows exactly from the construction $c^q(u,v)=c(u)$ and is confirmed numerically by zero prediction disagreement and zero logit difference.

Thus transformation decodability, invariance distance, and effective rank change under an exactly fixed class predictor.

\subsection{Case Study III: Domain Information Under Fixed Class Prediction}
\label{app:case-study-domain}

The third diagnostic targets domain or nuisance information. It uses PACS and OfficeHome, two domain-annotated image-classification datasets. The supervised task is object classification, while the domain label is treated as an auxiliary nuisance variable. By appending the one-hot domain label to the representation and forcing the class head to ignore it, the experiment tests whether domain information can be made perfectly decodable while the supervised class predictor remains unchanged.

\subsubsection{Setup}
\label{app:domain-setup}

Both datasets contain object-class labels and domain labels. The supervised task is object classification. The domain label is treated as an auxiliary nuisance variable and admissible side information for the diagnostic.

Using ImageNet initialization, a ResNet--18 classifier is fitted separately on each dataset. The trained classifier is written as
\begin{equation}
f(x)=c(h(x))
\end{equation}
where $h(x)\in\mathbb R^{512}$ is the penultimate-layer representation and $c$ is the final linear classification head. Training uses the class labels; domain labels are withheld from the supervised head.

The trained representation $h$ and head $c$ are then frozen. Let $q(X)$ denote the one-hot encoded domain label. Since both PACS and OfficeHome have four domains,
\begin{equation}
q(x)\in\{0,1\}^4.
\end{equation}
Define
\begin{equation}
h^q(x)=(h(x),q(x))\in\mathbb R^{516},
\end{equation}
and
\begin{equation}
c^q(u,v)=c(u).
\end{equation}
Therefore,
\begin{equation}
c^q(h^q(x))=c(h(x))
\end{equation}
for every input $x$.

Three groups of quantities are evaluated. First, predictor preservation is verified by measuring class accuracy, class loss, and prediction disagreement between $c\circ h$ and $c^q\circ h^q$. Second, domain accessibility is measured by training a linear probe to predict the domain label from either $h(X)$ or $h^q(X)$. Third, representation dimension and effective rank are reported as coarse compression-related diagnostics.

All results are reported over three random seeds. For the domain-probe gains, paired bootstrap confidence intervals are computed. Empirical power for detecting the probe-accuracy gain is also estimated under subsampling. The statistical analysis concerns the representation-level diagnostics; predictor preservation follows exactly from the construction.

\subsubsection{Predictor preservation}
\label{app:domain-predictor-preservation}

Table~\ref{tab:app-domain-predictor-preservation} verifies that the augmented systems preserve the supervised class predictor exactly. Class accuracy, class loss, predicted labels, and logits are unchanged. The prediction disagreement between $c\circ h$ and $c^q\circ h^q$ is zero on both datasets.

\begin{table}[h]
\centering
\caption{Predictor preservation under domain-label augmentation. The augmented representation is $h^q(x)=(h(x),q(x))$, where $q(x)$ is the one-hot domain label, and the augmented head is $c^q(u,v)=c(u)$. Values are mean $\pm$ standard deviation over three seeds.}
\label{tab:app-domain-predictor-preservation}
\begin{tabular}{lccc}
\toprule
Dataset & Class accuracy & Class loss & Prediction disagreement \\
\midrule
OfficeHome & $0.756 \pm 0.032$ & $0.942 \pm 0.079$ & $0.000 \pm 0.000$ \\
PACS       & $0.922 \pm 0.017$ & $0.222 \pm 0.049$ & $0.000 \pm 0.000$ \\
\bottomrule
\end{tabular}
\end{table}

The augmented representation explicitly contains the domain label, but the supervised head ignores the appended coordinates. Hence the class predictor is unchanged pointwise.

\subsubsection{Domain information becomes perfectly decodable}
\label{app:domain-decoding}

Table~\ref{tab:app-domain-decoding} reports domain-probe accuracy from the original representation $h$ and from the augmented representation $h^q$.

\begin{table}[h]
\centering
\caption{Domain-probe accuracy under fixed class prediction. Appending the one-hot domain label leaves the supervised class predictor unchanged, but makes domain information perfectly linearly decodable from the representation. Values are mean $\pm$ standard deviation over three seeds.}
\label{tab:app-domain-decoding}
\begin{tabular}{lccc}
\toprule
Dataset & Domain probe from $h$ & Domain probe from $h^q$ & Gain \\
\midrule
OfficeHome & $0.667 \pm 0.009$ & $1.000 \pm 0.000$ & $0.333 \pm 0.009$ \\
PACS       & $0.949 \pm 0.002$ & $1.000 \pm 0.000$ & $0.051 \pm 0.002$ \\
\bottomrule
\end{tabular}
\end{table}

The effect is strongest on OfficeHome. The original representation already contains nontrivial domain information, with domain-probe accuracy $0.667\pm0.009$. After predictor-preserving augmentation, domain-probe accuracy becomes $1.000\pm0.000$, giving a gain of $0.333\pm0.009$. On PACS, the original representation already makes domain highly accessible, with probe accuracy $0.949\pm0.002$; appending the domain label increases domain decodability to $1.000\pm0.000$.

Bootstrap confidence intervals for the domain-probe gains exclude zero for every seed. On OfficeHome, the seed-level gains are approximately $0.323$, $0.338$, and $0.338$. On PACS, the corresponding gains are approximately $0.050$, $0.050$, and $0.054$. Monte Carlo power estimates are $1.0$ for the one-hot augmentation on both datasets.

This witness is intentionally transparent: since the one-hot domain label is appended directly, perfect domain decodability is expected. Since the domain coordinate is appended explicitly, perfect probe recovery is expected. The relevant observation is that this change in domain decodability occurs while the class predictor is fixed pointwise. Supervised class prediction cannot certify that domain information is absent from the representation. A representation that explicitly contains the domain label and a representation that omits it are observationally equivalent from the viewpoint of supervised prediction when the head ignores the added coordinates.

\subsubsection{Compression-related diagnostics}
\label{app:domain-compression}

The same intervention also changes compression-related representation diagnostics. The original penultimate representation has dimension $512$. Since the domain label is represented by a four-dimensional one-hot vector, the augmented representation has dimension
\begin{equation}
512+4=516.
\end{equation}
Thus, the dimensionality of the representation increases from $512$ to $516$ by construction.

Table~\ref{tab:app-domain-effective-rank} reports the effective rank of the original and augmented representations. The effective rank also increases slightly, reflecting the additional domain coordinates, although the primary diagnostic effect is the change in domain decodability.

\begin{table}[h]
\centering
\caption{Representation dimension and effective rank under domain-label augmentation. The dimensionality increases from $512$ to $516$ by construction. Effective rank increases slightly, while the supervised class predictor is unchanged. Values are mean $\pm$ standard deviation over three seeds.}
\label{tab:app-domain-effective-rank}
\begin{tabular}{lcccc}
\toprule
Dataset & Dim. $h$ & Dim. $h^q$ & Eff. rank $h$ & Eff. rank $h^q$ \\
\midrule
OfficeHome & $512$ & $516$ & $73.621 \pm 3.239$ & $74.012 \pm 3.268$ \\
PACS       & $512$ & $516$ & $13.812 \pm 2.729$ & $13.898 \pm 2.739$ \\
\bottomrule
\end{tabular}
\end{table}

Effective rank is used as a coarse diagnostic of representation geometry rather than as a complete measure of information-theoretic compression. The result illustrates the same point as the domain-probe analysis: representation-level quantities can change while the supervised predictor is held fixed exactly.

Thus domain accessibility changes under an exactly fixed class predictor.

\subsection{Empirical conclusion}
\label{app:empirical-conclusion}

The three diagnostics instantiate the same mathematical obstruction in different representation-level languages. CelebA shows that semantic attributes can be made accessible while the supervised target predictor is fixed. CIFAR-10 and STL-10 show that transformation decodability and transformation sensitivity can change while the class predictor is fixed. PACS and OfficeHome show that domain information can be made perfectly decodable while the class predictor is fixed.

In every case, the equality
\begin{equation}
c^q\circ h^q=c\circ h
\end{equation}
holds by construction. The empirical measurements therefore show that the reported representation diagnostics can vary inside a fixed predictor fiber while prediction is unchanged.

\section{Additional Experimental Details}
\label{app:additional-experimental-details}

\subsection{Additional Waterbirds details}
\label{app:waterbirds-additional}

The Waterbirds constrained-representation experiment trains ERM and five
families of constrained models over ten seeds. For each non-ERM method and
seed, the configuration whose validation supervised performance is closest to
the corresponding ERM model is selected using validation accuracy and
cross-entropy. The matching tolerances are
\[
\Delta_{\rm acc}=0.01,
\qquad
\Delta_{\rm CE}=0.05 .
\]
The selected matched set contains $60$ seed-method models: ERM plus five
constraint families across ten seeds.

\begin{table}[h]
\centering
\caption{Paired changes relative to ERM after supervised-performance matching.
Differences are computed within seed. Negative invariance distance means greater
transformation stability.}
\label{tab:app-waterbirds-paired}
\small
\begin{tabular}{lccc}
\toprule
Comparison & Nuisance probe diff. & Invariance dist. diff. & Effective-rank diff. \\
\midrule
Bottleneck $-$ ERM & $-0.0657$ & $-0.109$ & $-33.8$ \\
VIB $-$ ERM & $-0.0322$ & $+0.0568$ & $-33.9$ \\
AugInv $-$ ERM & $+0.0064$ & $-0.0609$ & $-14.8$ \\
SupCon $-$ ERM & $+0.0058$ & $-0.0267$ & $-17.0$ \\
Adv $-$ ERM & $-0.0037$ & $-0.119$ & $-32.3$ \\
\bottomrule
\end{tabular}
\end{table}

The adversarial variant changes geometry and reduces effective rank, but in
this run it does not substantially reduce linear background decodability: the
paired nuisance-probe balanced-accuracy change is $-0.37$ percentage
points. The main text therefore emphasizes the other constrained models. The
result is still informative: an objective intended to affect nuisance information can
alter other representation properties without reliably removing the nuisance
variable under the probe used here.

\begin{figure}[h]
\centering
\includegraphics[width=0.78\linewidth]{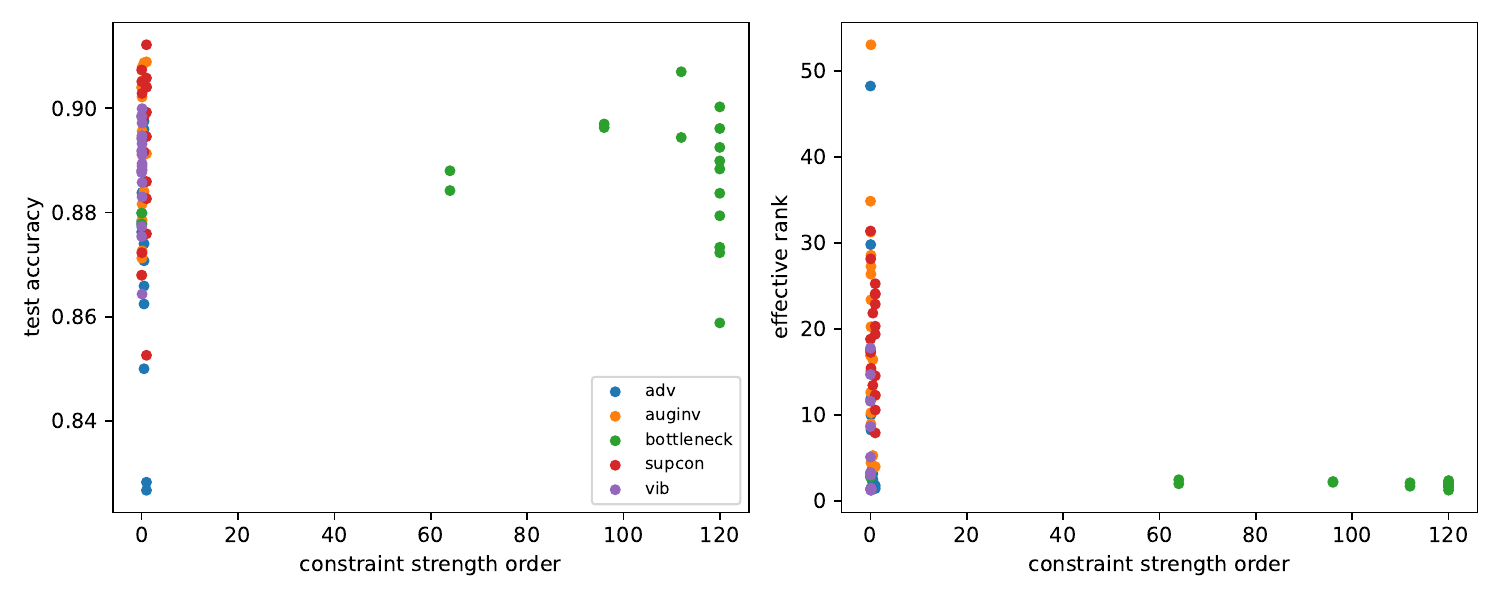}
\caption{Constraint paths on Waterbirds. Sweeping constraint strengths changes
task performance and representation diagnostics differently for different
objectives. The main text reports the supervised-performance-matched selections
from these paths.}
\label{fig:app-waterbirds-constraint-paths}
\end{figure}

\subsection{Predictor-preserving representation augmentation}
\label{app:repr-augmentation}

The main Waterbirds experiment uses naturally trained models. This appendix records
the exact algebraic witness that motivates the empirical design.

Let a supervised predictor factor as
\[
    f(x)=c(h(x)),
\]
where $h:\mathcal X\to\mathcal H$ is a representation and
$c:\mathcal H\to\mathcal A$ is the task head. For any measurable feature map
\[
    r:\mathcal X\to\mathcal R,
\]
define
\[
    \widetilde h(x)=(h(x),r(x)),
    \qquad
    \widetilde c(u,v)=c(u).
\]
Then, for every input $x$,
\[
    \widetilde c(\widetilde h(x))=c(h(x))=f(x).
\]
Thus every supervised quantity depending on the final prediction alone is
unchanged: pointwise predictions, task loss, task accuracy, calibration,
prediction disagreement, and predictive KL divergence with respect to another
fixed predictor are all identical.

Representation diagnostics, however, can change. If $r(x)$ contains a nuisance
variable, then a probe can recover nuisance information from $\widetilde h$
even if the task head ignores it. If $r(x)$ contributes high-variance or
high-dimensional coordinates, then the covariance spectrum, effective rank, and
CKA similarity of the representation can change. If $r(x)$ changes under a
transformation $g$, then the transformation distance between $\widetilde h(x)$
and $\widetilde h(gx)$ can increase, again without changing the supervised
predictor.

This construction is an exact finite-sample witness: supervised risk alone
cannot identify representation properties when the predictor is fixed
pointwise. The naturally trained Waterbirds experiment provides the
complementary observation that ordinary representation-level constraints select
different properties among models with comparable supervised behavior.

\subsection{Near-fiber thresholds and witnesses}
\label{app:near-fiber-details}

Table~\ref{tab:app-near-fiber-thresholds} records the exact thresholds used for
the same-seed near-fiber experiment.

\begin{table}[h]
\centering
\caption{Near-fiber thresholds and counts. Scalar selections are
median-centered. Pairwise selections use prediction disagreement and symmetric
KL.}
\label{tab:app-near-fiber-thresholds}
\small
\begin{tabular}{lcccccc}
\toprule
Dataset & $\varepsilon_A$ & $\varepsilon_L$ & Selected & $D_{\rm pred}$ & $D_{\rm KL}$ & Pair count \\
\midrule
Colored-MNIST & $0.005$ & $0.005531$ & $25/50$ & $0.02$ & $0.10$ & $132$ all-model; $8$ selected \\
CelebA & $0.005$ & $0.009970$ & $13/25$ & $0.05$ & $0.09$ & $96$ all-model; $21$ selected \\
\bottomrule
\end{tabular}
\end{table}

\begin{table}[h]
\centering
\caption{Representative near-prediction-equivalent witnesses. The predictors
are close but the representations are not identical under CKA, Procrustes, probe
gaps, and holonomy gaps.}
\label{tab:app-near-fiber-witnesses}
\scriptsize
\begin{tabular}{llcccccc}
\toprule
Dataset & Pair & $D_{\rm pred}$ & $D_{\rm KL}$ & CKA & Procrustes & Probe gap & Holonomy gap \\
\midrule
Colored-MNIST & \texttt{005}/\texttt{031} & $0.0184$ & $0.07455$ & $0.8767$ & $0.4734$ & $0.0081$ & $0.02235$ \\
Colored-MNIST & \texttt{005}/\texttt{026} & $0.0191$ & $0.08474$ & $0.8869$ & $0.4690$ & $0.0015$ & $0.01819$ \\
Colored-MNIST & \texttt{028}/\texttt{031} & $0.0186$ & $0.08431$ & $0.8624$ & $0.4897$ & $0.0091$ & $0.01391$ \\
CelebA & \texttt{005}/\texttt{022} & $0.0487$ & $0.08510$ & $0.9101$ & $0.3955$ & $0.0258$ & $0.00438$ \\
CelebA & \texttt{015}/\texttt{022} & $0.0468$ & $0.08561$ & $0.9072$ & $0.3929$ & $0.0253$ & $0.00434$ \\
CelebA & \texttt{000}/\texttt{004} & $0.0465$ & $0.07592$ & $0.9229$ & $0.3603$ & $0.0108$ & $0.00257$ \\
\bottomrule
\end{tabular}
\end{table}

\begin{figure}[h]
\centering
\begin{minipage}{0.49\linewidth}
\centering
\includegraphics[width=\linewidth]{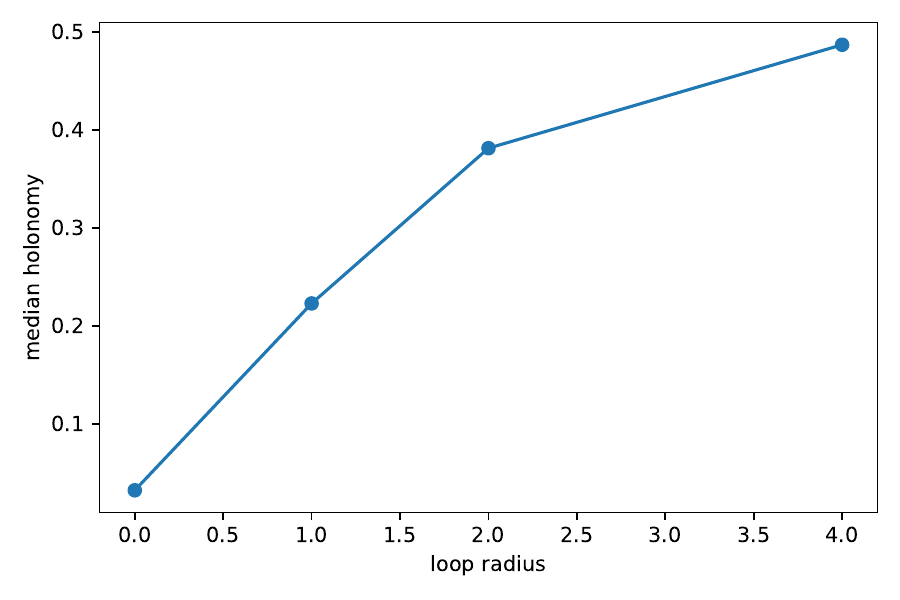}\\
{\small (a) Colored-MNIST}
\end{minipage}
\hfill
\begin{minipage}{0.49\linewidth}
\centering
\includegraphics[width=\linewidth]{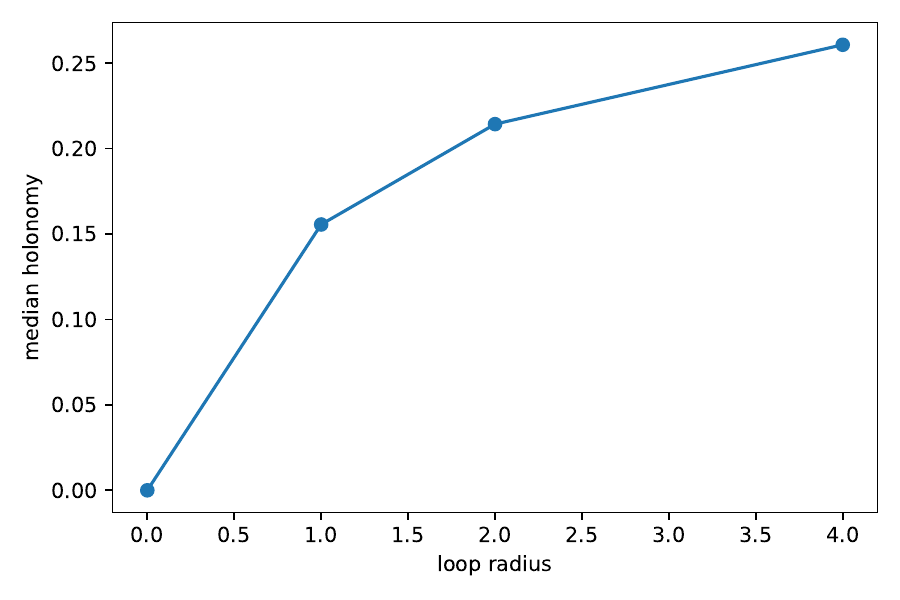}\\
{\small (b) CelebA}
\end{minipage}
\caption{Holonomy versus loop radius. The pathwise diagnostic is supporting
evidence. It is strongest on
Colored-MNIST and weaker on CelebA.}
\label{fig:app-holonomy-radius}
\end{figure}

The near-fiber experiment has a narrower interpretation than the exact
predictor-preserving witnesses. These selections are approximate predictor
fibers. For the Colored-MNIST holonomy diagnostic\citep{sevetlidis2026gaugeinvariant}, several individual
zero-radius rows are above numerical floor even though the aggregate numerical
check passes at the median scale. The main text therefore relies primarily on
supervised matching, prediction distance, CKA, Procrustes, RSA, effective-rank,
and probe diagnostics, while using holonomy as secondary pathwise evidence.

\end{document}